\documentclass[11pt,letterpaper]{article}

\usepackage[margin=1in]{geometry}
\usepackage{amsmath,amssymb,amsthm}
\usepackage{graphicx}
\usepackage{subcaption}
\usepackage{tikz}
\usetikzlibrary{positioning,arrows.meta,calc}
\usepackage{booktabs}
\usepackage{multirow}
\usepackage{array}
\usepackage{xcolor}
\usepackage{url}
\usepackage[american]{babel}
\usepackage{csquotes}
\usepackage[style=apa,backend=biber,natbib=true]{biblatex}
\usepackage{hyperref}
\usepackage{setspace}
\usepackage{enumitem}
\usepackage{algorithm}
\usepackage{algpseudocode}

\hypersetup{
  colorlinks=true,
  linkcolor=blue,
  citecolor=blue,
  urlcolor=blue
}

\newcommand{\E}{\mathbb{E}}
\newcommand{\Prob}{\mathbb{P}}
\newcommand{\Var}{\mathrm{Var}}

\newcommand{\diag}{\mathrm{diag}}
\newcommand{\sg}{\mathrm{sg}}
\newcommand{\R}{\mathbb{R}}
\newcommand{\cS}{\mathcal{S}}
\newcommand{\cC}{\mathcal{C}}
\newcommand{\cF}{\mathcal{F}}
\newcommand{\cD}{\mathcal{D}}

\newcommand{\one}{\mathbf{1}}

\newtheorem{definition}{Definition}
\newtheorem{assumption}{Assumption}
\newtheorem{theorem}{Theorem}
\newtheorem{proposition}{Proposition}

\newtheorem{remark}{Remark}

\title{General Value Functions for Remaining Useful Life and Failure-Mode Prediction}

\author{%
  Hao Yan$^{1}$\thanks{Corresponding author: \textit{haoyan@asu.edu}.},
  Ali Sarabi$^{1}$,
  Qing Zou$^{1}$\thanks{Qing Zou conducted the research while he was a visiting scholar at Arizona State University (\textit{qingzou@asu.edu}).},
  Boyang Xu$^{1}$\\
  $^{1}$School of Computing and Augmented Intelligence, Arizona State University}

\date{}

\begin{document}
\setstretch{1.5}
\maketitle

\begin{abstract}
Remaining useful life (RUL) prediction and failure-mode classification are central tasks in predictive maintenance. Many data-driven pipelines use fixed-window supervised learning with complete terminal labels; such routes do not naturally encode the temporal recursion linking successive degradation-state predictions when observations are partial or unit identities are unavailable. We formulate prognostics as vector General Value Function (GVF) prediction on an absorbing degradation process, treating RUL and failure-mode probabilities as temporally consistent targets rather than independent window-level labels, and estimate them with a multi-step temporal-difference estimator, TD($n,\lambda$). Supporting theory identifies the Bellman fixed point of the vector GVFs, characterizes the linear projected-TD limit and its relation to complete-return Monte Carlo regression under realizability, and explains when bootstrapped TD targets are less variable than Monte Carlo returns. On an event-triggered multimode simulation and NASA C-MAPSS label-scarce stitch data, TD improves RUL and failure-mode prediction relative to a supervised same-backbone Monte Carlo control, especially under scarce complete labels. Practically, fragmented, identity-free degradation records can contribute local Bellman transitions instead of being discarded until complete run-to-failure labels are available.
\end{abstract}



\section{Introduction}
\label{sec:introduction}

Modern engineered systems are increasingly monitored through streams of sensor measurements---vibration, temperature, pressure, flow, emissions, and related condition indicators. These signals evolve as systems operate, age, and degrade. Prognostics uses such trajectories to forecast reliability outcomes, especially remaining useful life (RUL) and, when several failure mechanisms are possible, the likely failure mode. Accurate prognostic predictions support condition-based maintenance by allowing service actions to be planned before unexpected breakdowns \citep{chen2013condition, song2018statistical, li2007failure, zhou2014remaining}.

A large literature has studied RUL and failure-mode prediction through statistical degradation models, health-index construction, and supervised machine and deep learning \citep{kontar2017rulprediction, liu2013data, Li2018remaining, fu2025degradation}. Indirect health-index approaches provide interpretable degradation representations but depend on the validity of the selected surrogate dynamics, thresholds, and extrapolation model. Direct supervised approaches map sensor histories to RUL or terminal-mode labels, but often represent successive windows as separate training examples and impose temporal coherence only through architecture or auxiliary penalties. Both families are most effective when abundant run-to-failure trajectories and terminal labels are available; Section~\ref{sec:literature_review} reviews these approaches in detail.

This paper instead foregrounds the temporal structure of the prognostic target. RUL and failure-mode probabilities evolve along a degradation trajectory rather than forming independent window-level labels. Conditional on continued operation, realized RUL decreases by one operating step between successive states. Failure-mode probabilities may change as new evidence arrives, but their updates should remain consistent with the future evolution implied by the successor state. Existing methods often encourage this behavior indirectly through monotonic health indices, smoothness penalties, or successive-window consistency terms \citep{liu2013data, kontar2017rulprediction, jahani2020remaining, fu2025degradation}. We instead define temporal semantics directly at the target level through hitting times and absorption probabilities.

At each time, prediction is conditioned on a \emph{predictive state} constructed from the available observation history---for example, a sensor vector, a recent history window, a recurrent summary, or a learned health index---rather than on an unobserved physical state. The unit evolves through this predictive-state space until reaching an absorbing failure set $\cF$. RUL is the expected time until the trajectory first reaches $\cF$. When the terminal boundary is partitioned as
$\cF=\bigcup_{k=1}^{K}\cF_k$,
failure-mode prediction asks for the probability of eventual absorption in each mode-specific set $\cF_k$. RUL and failure-mode classification are therefore related questions about the same future terminal event, rather than unrelated regression and classification tasks. This perspective separates the construction of the predictive state from the definition of the prognostic target: different sensor representations may be used without changing the hitting-time and hitting-probability quantities being predicted.

These quantities share a forward-looking structure: each prediction summarizes future evolution until a terminal event occurs. \emph{General value functions} (GVFs) provide a compact language for defining such temporally extended predictions \citep{sutton2011horde, white2015developing}. In the present setting, value functions are used purely for prediction. There are no control actions, maintenance policies, or optimization over decisions. RUL is represented by accumulating operating time until absorption, while each failure-mode quantity accumulates a mode-specific terminal event.

The GVF formulation also supplies the desired temporal consistency through a Bellman equation. A prediction at the current state must agree with the immediate transition and the prediction made from the successor state. For RUL, conditional on survival, remaining life equals one elapsed step plus the expected remaining life from the next state. For failure modes, mode-specific hitting probabilities propagate through successor states until absorption. The Bellman relation is therefore used as the consistency equation for dynamic prognosis.
Figure~\ref{fig:rul_gvf_analogy} illustrates this shared
hitting-time and hitting-probability interpretation.

\begin{figure}[htbp]
\centering
\includegraphics[width=0.6\linewidth]{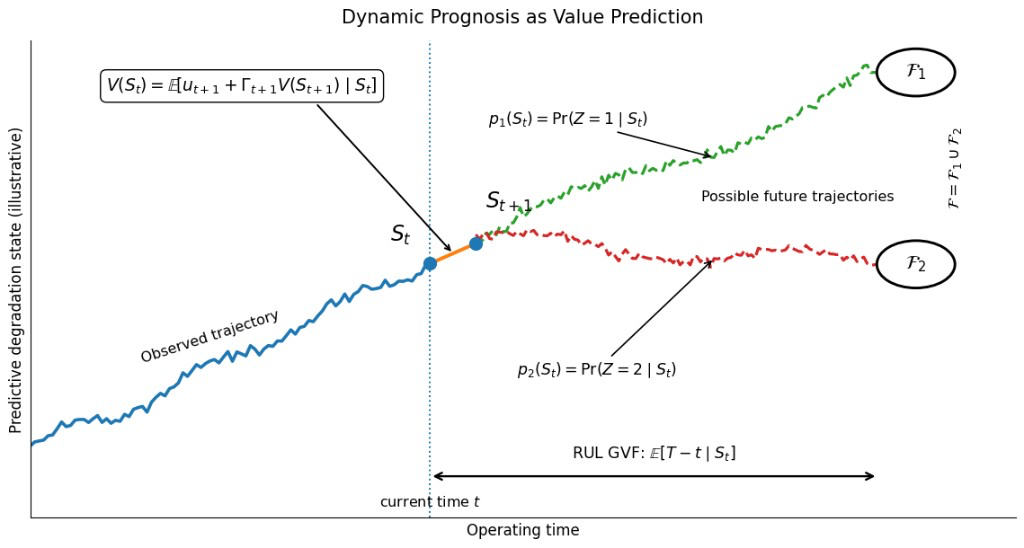}
\caption{\textbf{Dynamic prognosis as value prediction.} A unit evolves until absorption in $\cF=\bigcup_k\cF_k$; RUL is the expected time to failure, and mode prediction is the hitting probability of each $\cF_k$. Successive forecasts satisfy a Bellman recursion along the absorbing degradation process.}
\label{fig:rul_gvf_analogy}
\end{figure}

This viewpoint clarifies the distinction between \emph{Monte Carlo} (MC) and \emph{temporal-difference} (TD) estimation. Conventional supervised RUL regression uses complete run-to-failure labels, corresponding to complete Monte Carlo returns. Such targets are effective when terminal labels are abundant, but they cannot be formed directly for censored, incomplete, or anonymous nonterminal trajectory segments. TD learning instead combines observed local transitions with bootstrapped successor predictions. Partial trajectories can therefore contribute Bellman-consistent training targets before the eventual failure time or mode is observed. We develop a multi-step TD($n,\lambda$) estimator for the resulting vector GVF.

The failure boundary and its mode-specific subsets may be defined
operationally from physical thresholds, benchmark labels, maintenance
records, or expert annotations.

The contributions of this paper are as follows.
\begin{enumerate}[leftmargin=*]
\item We formulate dynamic prognostics as vector-GVF prediction on an absorbing degradation process, unifying RUL and failure-mode prediction as temporally extended hitting-time and hitting-probability questions rather than independent supervised tasks. We further develop a TD($n,\lambda$) estimator that enables bootstrapped learning from partial trajectories.
\item Drawing on standard TD analysis, we identify the Bellman fixed point targeted by the vector GVFs, characterize the linear projected-TD limit and its relation to complete-return Monte Carlo regression under realizability, and provide a target-variance decomposition explaining when bootstrapped TD targets can be less variable than complete Monte Carlo returns.
\item We evaluate the method on an event-triggered multimode simulation and the NASA C-MAPSS benchmark, emphasizing heterogeneous degradation, scarce terminal labels, and anonymous partial-trajectory observations.
\end{enumerate}


\section{Literature Review}
\label{sec:literature_review}

We distinguish \emph{direct} and \emph{indirect} methods by the prognostic supervision target---whether RUL or failure risk is inferred directly from observations or through an intermediate health representation---rather than by whether the method is statistical, model-based, or deep learning-based.

\subsection{Direct Prognostic Methods}
\label{sec:direct_prognostics}

In this paper, a \emph{direct prognostic method} is one whose learning or estimation target is the prognostic outcome itself---e.g., a pointwise RUL label $Y_{it}=T_i-t$, an event-time and censoring pair $(T_i,\delta_i)$, a failure-mode label, or a terminal degradation outcome---regardless of whether the model is statistical or neural. Survival and lifetime models \citep{li2007failure, yue2021joint} estimate failure risk or residual life from event-time data. Stochastic degradation-process \citep{ye2014inverse, fang2024class, vannoortwijk2009survey}, mixed-effects \citep{kontar2017rulprediction}, and Bayesian \citep{gebraeel2005residual} approaches may posit an explicit degradation law, but remain direct when parameters are fit to observed paths and the inferential target is RUL or time to failure; a parametric degradation structure alone does not make a method indirect.

Data-driven direct methods map condition-monitoring measurements or engineered features to RUL, failure time, or mode labels, including recurrent networks \citep{zheng2017long}, support-vector regressors \citep{khelif2017direct}, and Bayesian deep regressors \citep{kim2020bayesian}. Modern deep networks minimize supervised losses on RUL, capped RUL, or failure outputs using CNNs \citep{Li2018remaining}, Transformer architectures \citep{li2022domain}, and related deep prognostics frameworks \citep{zhao2019deep}.

Direct methods optimize the deployed prognostic quantity and are easy to benchmark, which helps explain their strong performance with end-to-end deep learning. They nonetheless require sufficient labeled degradation histories---and window-level training can induce strong within-unit correlation---while survival and degradation models rely on hazard, random-effect, or threshold assumptions. Even successful deep direct models can enforce temporal coherence through architecture or auxiliary penalties rather than through the definition of the prediction target.

\subsection{Indirect Prognostic Methods}
\label{sec:indirect_prognostics}

We use the term \emph{indirect prognostic method} for approaches that first construct an intermediate degradation representation---a health index, latent state, fused signal, or failure surface---and then derive RUL or failure risk by extrapolating, thresholding, or passing it to a downstream reliability model. Health-index methods fuse multivariate sensor streams into scalar or low-dimensional indicators \citep{liu2013data}.

Indirect routes offer interpretable degradation coordinates and support multisensor fusion across operating regimes \citep{yan2016}, composite health-index modeling \citep{song2018statistical}, prognostics under multiple failure modes \citep{chehade2018data, fu2025degradation}, and failure-surface frameworks \citep{song2019IISE}. Deep encoders are increasingly used for health-index construction via RNN indicators \citep{guo2017recurrent} and shape-constrained fusion \citep{li2021shape}.

The main limitation is dependence on this intermediate representation: a poorly specified health index can bias downstream RUL even when the second-stage model is sound, and choices about fusion, threshold, and extrapolation can propagate misspecification through the pipeline.

\subsection{Multi-Task Prognostics}
\label{sec:multitask_review}

Traditional pipelines often treat RUL regression and failure-mode classification as separate stages, although reliability methods can couple them through joint time-series--survival models \citep{yue2021joint}, failure-surface formulations \citep{song2019IISE}, and degradation modeling under unknown failure modes \citep{fu2025degradation}. Failure modes are not terminal labels alone; they reflect distinct degradation mechanisms, rates, and sensor signatures along the same trajectory.

Multi-task learning \citep{caruana1997multitask} shares representations across related tasks to improve generalization. In prognostics, shared encoders have been used for joint degradation assessment and RUL prediction \citep{miao2019joint}, and for mode-dependent RUL with failure-mode diagnostics \citep{li2022deepbranched}.

Many supervised multi-task models nonetheless combine regression and classification losses without specifying temporal semantics: terminal classifiers need not be absorption probabilities, and RUL heads need not satisfy Bellman consistency across states.

\subsection{General Value Functions and Temporal-Difference Learning}
\label{sec:gvf_td_review}

Temporal-difference learning estimates temporally extended predictions
from bootstrapped successor values rather than complete returns
\citep{sutton1988learning,sutton2018reinforcement}. Relevant supporting
work includes convergence under function approximation
\citep{tsitsiklis1997analysis}, statistical analyses of TD targets
\citep{cheikhi2023statistical}, and rare-event prediction
\citep{cheng2024surprising}. General value functions extend value
prediction through user-specified cumulants and continuations
\citep{sutton2011horde,white2015developing}.
Their use for joint RUL and failure-mode prognosis under fragmented
trajectory supervision remains comparatively unexplored.

\section{Problem Formulation: Absorbing Degradation Processes}
\label{sec:problem_formulation}

We formalize RUL and failure-mode prediction as vector GVFs on an
absorbing degradation process.

\subsection{State, Dynamics, and Absorption}
\label{sec:state_dynamics_absorption}

Let $S_t\in\cS\subseteq\R^d$ denote the predictive state at time $t$,
constructed from the observation history available for forecasting.
The Bellman equations below are exact when $\{S_t\}$ is Markov for the
prognostic targets and approximate otherwise; predictions condition on
$S_t$ rather than on an unobserved physical state.

\begin{definition}[Absorbing degradation process]
\label{def:absorbing_degradation_process}
We model $\{S_t\}$ as a Markov process on $\cS$ with kernel
$P(\cdot\mid s)$. Partition $\cS=\cC\cup\cF$, where
$\cC=\cS\setminus\cF$ is the nonterminal operating region and $\cF$ is
an absorbing failure set, with $P(\cF\mid s)=1$ for $s\in\cF$. Define
the absorption time
\[
T=\inf\{t\ge0:S_t\in\cF\}.
\]
\end{definition}

\begin{definition}[Relative hitting time]
\label{def:relative_hitting_time}
The relative hitting time from $S_t$ is
\[
\tau_t=\inf\{h\ge0:S_{t+h}\in\cF\}.
\]
On paths with finite absorption, $\tau_t=T-t$.
\end{definition}

\begin{definition}[Failure modes]
\label{def:failure_modes}
When failures decompose into $K$ modes, partition the terminal set into
disjoint subsets $\cF_1,\ldots,\cF_K$:
\[
\cF=\bigcup_{k=1}^K\cF_k,
\qquad
\cF_j\cap\cF_k=\emptyset
\quad (j\ne k).
\]
Absorption in $\cF_k$ yields terminal mode $Z=k$.
\end{definition}

Several results in Section~\ref{sec:algorithmic_properties} additionally
assume almost-sure absorption,
$\Prob_s(\tau_t<\infty)=1$ for every $s\in\cC$; this assumption is stated
where required.

We adopt a model-free viewpoint with respect to $P$: the learning
algorithms use observed transitions $(S_t,S_{t+1})$ rather than a
closed-form plant model \citep{sutton2018reinforcement}.

Constructing the predictive state $S_t$ is distinct from specifying the
GVF targets on $(S_t,P)$. The Bellman equations are exact when $S_t$ is
Markov for those targets and otherwise define an approximation relative
to the selected representation.

\begin{remark}[Predictive state and partial observability]
In practice, $S_t$ is generally an approximate history summary rather
than the true physical state. The resulting setting is analogous to
partial observability: Bellman consistency is interpreted relative to
the selected predictive representation, not to an exact latent-state
model \citep{xu2025partially}.
\end{remark}

\subsection{Introduction to General Value Functions}
\label{sec:gvf_introduction}

\emph{General value functions} (GVFs) are defined in a setting-agnostic way: the only probabilistic object is a Markov process $\{S_t\}$ on $\cS$ with kernel $P(\cdot\mid s)$ \citep{sutton2011horde, white2015developing}. Each GVF is specified by cumulant and continuation sequences as the expected accumulated signal along the trajectory. Nothing in the definitions below requires actions, rewards in the reinforcement-learning sense, or optimality---only the process $(S_t,P)$ and the chosen cumulant and continuation sequences.

Fix $m\ge 1$ GVFs on the same $(S_t,P)$. For each $i\in\{1,\ldots,m\}$, specify an adapted \emph{cumulant} $c^{(i)}_{t+1}$ (a deterministic function of the transition when written in lowercase) and a \emph{continuation} $\gamma_{t+1}^{(i)}\in[0,1]$, which names the effective horizon for the $i$th GVF ($\gamma_{t+1}^{(i)}=0$ stops accumulation at the next step; constant $\gamma_{t+1}^{(i)}=\bar\gamma<1$ yields geometric discounting; state-dependent continuations encode pseudo-termination) \citep{sutton2018reinforcement}. In prognostics, for example, one may take $m=1+K$ to predict RUL together with $K$ failure-mode hitting probabilities.

\begin{definition}[Vector GVF]
\label{def:gvf_vector}
Stack the corresponding $m$ GVFs into
\[
V(s)=\bigl(V^{(1)}(s),\ldots,V^{(m)}(s)\bigr)^{\top}\in\R^m,
\]
where the $i$th component is
\begin{equation}
V^{(i)}(s)
=
\E\!\left[
\sum_{\ell=0}^{\infty}
\left(\prod_{j=1}^{\ell}\gamma_{t+j}^{(i)}\right)c_{t+\ell+1}^{(i)}
\;\middle|\; S_t=s
\right],
\label{eq:scalar_gvf_component}
\end{equation}
with the convention that the empty product equals one.
\end{definition}

Equation~\eqref{eq:scalar_gvf_component} is the \emph{extensive} definition: each $V^{(i)}(s)$ is an expectation of a return built from future cumulants gated by future continuations. When $m=1$ and $c_{t+1}$ coincides with a Markov reward signal, the single component reduces to the usual discounted value function. Collect instantaneous cumulants in $u_{t+1}=(c^{(1)}_{t+1},\ldots,c^{(m)}_{t+1})^{\!\top}$ and continuations in $\Gamma_{t+1}=\diag(\gamma_{t+1}^{(1)},\ldots,\gamma^{(m)}_{t+1})$. Under the Markov property, $V$ admits the \emph{vector Bellman expectation equation}
\begin{equation}
V(s)=\E\left[u_{t+1}+\Gamma_{t+1}V(S_{t+1})\mid S_t=s\right],
\label{eq:vector_bellman}
\end{equation}
together with componentwise boundary values on any terminal states where a return is pinned by definition. When $m=1$,~\eqref{eq:vector_bellman} reduces to the scalar Bellman equation $V(s)=\E[c_{t+1}+\gamma_{t+1}V(S_{t+1})\mid S_t=s]$, which is the relation bootstrapped by temporal-difference learning. Diagonal $\Gamma_{t+1}$ means each GVF component has its own horizon schedule, while $\Gamma_{t+1}V(S_{t+1})$ couples the GVFs only through the shared successor $S_{t+1}$.

\subsection{Multi-Task Prognostic Model: Two GVF Types on the Same Absorbing Degradation Process}
\label{sec:vector_gvf}

RUL prediction and
failure-mode prediction may appear as two unrelated supervised tasks, but both are
GVFs on the same stochastic process $(S_t,P)$. They differ only in cumulants,
continuations, and terminal semantics.

\begin{definition}[RUL (survival-time) GVF]
\label{def:survival_time_rul_gvf}
For $\gamma_{\mathrm{time}}\in[0,1]$, the \emph{survival-time GVF} assigns cumulant
$c_{t+1}^{\mathrm{time}}=\one\{S_t\notin\cF\}$ and continuation
$\Gamma_{t+1}^{\mathrm{time}}=\gamma_{\mathrm{time}}\one\{S_{t+1}\notin\cF\}$. For
$s\in\cC$,
\begin{equation}
V_{\gamma_{\mathrm{time}}}^{\mathrm{time}}(s)
=
\E_s\!\left[\sum_{\ell=0}^{\tau_t-1}\gamma_{\mathrm{time}}^\ell\right].
\label{eq:survival_time_gvf_main}
\end{equation}
When $\gamma_{\mathrm{time}}=1$ and $\E_s[\tau_t]<\infty$, cycle-count RUL is
$V_{1}^{\mathrm{time}}(s)=\E_s[\tau_t]$.
\end{definition}

For $\gamma_{\mathrm{time}}<1$, $V_{\gamma_{\mathrm{time}}}^{\mathrm{time}}(s)$ is a
discounted survival-time value with soft effective horizon approximately
$(1-\gamma_{\mathrm{time}})^{-1}$. Soft-horizon bounds are given in
Section~\ref{sec:gvf_readouts}.

\begin{definition}[Mode-$k$ terminal-event GVF]
\label{def:mode_terminal_event_gvf}
For each failure mode $k\in\{1,\ldots,K\}$ and
$\gamma_{\mathrm{mode}}\in[0,1]$, the \emph{mode-$k$ terminal-event GVF} assigns cumulant
$c_{t+1}^{(k)}=\one\{S_t\notin\cF,\ S_{t+1}\in\cF_k\}$ and continuation
$\Gamma_{t+1}^{(k)}=\gamma_{\mathrm{mode}}\one\{S_{t+1}\notin\cF\}$. For $s\in\cC$,
\begin{equation}
H_{\gamma_{\mathrm{mode}}}^{(k)}(s)
=
\E_s\!\left[
\gamma_{\mathrm{mode}}^{\tau_t-1}\one\{S_{t+\tau_t}\in\cF_k\}
\right].
\label{eq:mode_gvf_main}
\end{equation}
When $\gamma_{\mathrm{mode}}=1$, $H_1^{(k)}(s)=\Prob_s(Z=k)$.
\end{definition}

When $\gamma_{\mathrm{mode}}<1$, each $H_{\gamma_{\mathrm{mode}}}^{(k)}(s)$ is a discounted terminal-event GVF and need not equal $\Prob_s(Z=k)$; the urgency-reweighted interpretation of the normalized scores $\tilde p_{\gamma,k}$ is developed in Proposition~\ref{prop:terminal_event_mode_tilting}.

\paragraph{Vector Bellman equation.}
The joint prognostic GVF stacks survival-time and mode-specific components as
$V_\gamma(s)=(V_\gamma^{\mathrm{time}}(s),H_\gamma^{(1)}(s),\ldots,H_\gamma^{(K)}(s))^\top$.
It specializes Eq.~\eqref{eq:vector_bellman} with vector cumulant
$u_{t+1}^{(\gamma)}=(\one\{S_t\notin\cF\},\one\{S_t\notin\cF,S_{t+1}\in\cF_1\},\ldots,
\one\{S_t\notin\cF,S_{t+1}\in\cF_K\})^\top$ and continuation
$\Gamma_{t+1}^{(\gamma)}=\gamma\one\{S_{t+1}\notin\cF\}I_{K+1}$, i.e.,
\begin{equation}
V_\gamma(s)
=
\E\!\left[
u_{t+1}^{(\gamma)}
+
\Gamma_{t+1}^{(\gamma)}
V_\gamma(S_{t+1})
\mid S_t=s
\right].
\label{eq:prognostic_vector_bellman_main}
\end{equation}
In implementation, components need not share the same discount; for example,
$\gamma_{\mathrm{time}}<1$ for a soft-horizon RUL proxy and
$\gamma_{\mathrm{mode}}=1$ for probability-correct terminal-mode prediction.


\section{Learning Algorithm}
\label{sec:learning_algorithm}

\subsection{TD($n,\lambda$) Estimator}
\label{sec:td_estimator}

Two classical ways to estimate a value function are Monte Carlo (MC) estimation and temporal-difference (TD) estimation. In the present vector-GVF setting, both operate componentwise on the same stacked prediction vector; the difference is whether the target waits for a completed return or bootstraps from a later prediction.

\paragraph{Monte Carlo targets.}
MC estimation uses the realized return from a complete trajectory. For RUL, this is exactly the familiar supervised label $T-t$ in the undiscounted case; for a mode-probability GVF, the realized terminal return is the one-hot indicator of the failure mode. Thus conventional window-based RUL regression can be interpreted as an every-visit MC method: each windowed state $S_t$ is paired with the complete return observed from that point to failure. This interpretation is useful because it explains both the strength and the weakness of supervised RUL labels. MC targets are direct samples of the desired return and do not depend on the current network prediction, but they require terminal information and can have high variance when the remaining episode length $\tau_t$ is long or failure events are rare.

For the vector GVF, the MC target from time $t$ is
\begin{equation}
y_t^{\mathrm{MC}}
=
\sum_{\ell=0}^{\tau_t-1}
\left(\prod_{j=1}^{\ell}\Gamma_{t+j}\right)u_{t+\ell+1},
\label{eq:vector_mc_target}
\end{equation}
with the convention that an empty product equals the identity matrix. If the trajectory is fully observed until absorption, all components of $y_t^{\mathrm{MC}}$ are available. If a trajectory is censored or only partially observed, some MC components are unavailable because the terminal event or terminal mode has not been observed; those components must be masked or replaced by a bootstrapped target.

\paragraph{TD($n$) targets.}
TD learning replaces the full future return by a shorter observed prefix plus a value estimate at a future state. This is the bootstrapping step: instead of waiting until failure, the learner uses the current approximation of the remaining return after $n$ transitions. For prognostics, this matters because transition data are often available long before an asset fails. TD updates can therefore use partial trajectories and can propagate temporal information backward through the trajectory incrementally.

The $n$-step TD target for vector GVFs is
\begin{equation}
y_t^{(n)}
=
\sum_{\ell=0}^{n-1}
\left(\prod_{j=1}^{\ell}\Gamma_{t+j}\right)u_{t+\ell+1}
+
\left(\prod_{j=1}^{n}\Gamma_{t+j}\right)
\sg\!\left(V_\theta(S_{t+n})\right),
\label{eq:vector_tdn_target}
\end{equation}
where $\sg(\cdot)$ denotes the stop-gradient operator. The stop-gradient notation emphasizes the semi-gradient nature of the update: the bootstrap target is treated as a fixed target for the current optimization step, rather than allowing gradients to flow through both sides of the Bellman equation.

Under the absorbing convention of Definition~\ref{def:absorbing_degradation_process}, both cumulant and continuation vanish once the state enters $\cF$ (each carries the indicator $\one\{S_t\notin\cF\}$ or $\one\{S_{t+1}\notin\cF\}$), so all post-terminal cumulants contribute zero and the continuation product zeroes out the bootstrap after absorption. Consequently, if the trajectory terminates before $t+n$, the product of continuations terminates the bootstrap term, so the TD target automatically reduces to the observed terminal return over the available prefix. If $n=1$, Eq.~\eqref{eq:vector_tdn_target} is the one-step TD target $u_{t+1}+\Gamma_{t+1}V_\theta(S_{t+1})$. As $n$ increases toward the remaining episode length, the target approaches the MC return in Eq.~\eqref{eq:vector_mc_target}. Thus $n$ controls the usual bias--variance and locality--horizon trade-off: small $n$ uses highly local bootstrapping with lower target variance but greater dependence on the current approximation, whereas large $n$ uses more observed future information but becomes closer to high-variance complete-return supervision.


For the decomposed $p_k,q_k$ parameterization (mode probabilities plus mode-specific RUL contributions), the same target equations apply after replacing the stacked vector $u_{t+1}$ with cumulants for the probability components and mode-specific RUL-contribution components. The main difference is interpretive: the $p_k$ components learn terminal hitting probabilities, while the $q_k$ components learn the portion of expected remaining time assigned to trajectories that eventually terminate in mode $k$.


\paragraph{TD($n,\lambda$).}
We use truncated TD($n,\lambda$), a generalization of fixed-step TD($n$), as the primary estimator throughout our main experiments.

\begin{remark}[TD target validity under censoring]
\label{rem:td_censoring_validity}
Fix $t$ and a prefix observed through $t+n$, with absorption time $T$ possibly unobserved. The MC target $y_t^{\mathrm{MC}}$ (Eq.~\eqref{eq:vector_mc_target}) is well-defined from this prefix only if $T \le t+n$; otherwise its components are undefined and must be masked. The TD($n$) target $y_t^{(n)}$ (Eq.~\eqref{eq:vector_tdn_target}) is well-defined whenever $t+n$ lies in the prefix: if $T \le t+n$ it reduces to the terminal return; if $T > t+n$, the bootstrap $\sg(V_\theta(S_{t+n}))$ completes the target from the observed successor state alone.
\end{remark}

Target-construction computability is not arbitrary-censoring consistency: observed prefixes must still represent the transition distribution, or sampling/weighting must address selection.

For $k\in\{1,\ldots,n\}$, let $y_t^{(k)}$ denote the $k$-step vector TD target from Eq.~\eqref{eq:vector_tdn_target} with $n$ replaced by $k$. Define the truncated TD($n,\lambda$) target
\begin{equation}
y_t^{(n,\lambda)}
=(1-\lambda)\sum_{k=1}^{n-1}\lambda^{k-1}y_t^{(k)}+\lambda^{n-1}y_t^{(n)},
\qquad \lambda\in[0,1],
\label{eq:vector_tdnl_target}
\end{equation}
which recovers one-step TD at $\lambda=0$ and TD($n$) at $\lambda=1$. The trace parameter $\lambda$ sets the effective temporal span of the target, emphasizing short bootstraps for small $\lambda$ and longer observed returns for large $\lambda$.

Partition the vector TD$(n,\lambda)$ target into its survival-time and
mode-GVF components,
\[
y_t^{(n,\lambda)}
=
\left(
y_{t,\mathrm{time}}^{(n,\lambda)},
y_{t,\mathrm{mode}}^{(n,\lambda)}
\right).
\]
The two heads regress the corresponding blocks of this common vector
Bellman target:
\begin{align}
\mathcal L_{\mathrm{rul}}(\theta)
&=
\left(
V_\theta^{\mathrm{time}}(S_t)
-
\sg\!\left(y_{t,\mathrm{time}}^{(n,\lambda)}\right)
\right)^2,
\label{eq:weighted_loss}\\
\mathcal L_{\mathrm{cls}}(\theta)
&=
\frac{1}{K}
\left\|
\widehat H_\theta(S_t)
-
\sg\!\left(y_{t,\mathrm{mode}}^{(n,\lambda)}\right)
\right\|_2^2.
\label{eq:class_loss}
\end{align}
The shared encoder and two heads are trained jointly using
\begin{equation}
\mathcal L(\theta)
=
\lambda_{\mathrm{rul}}\mathcal L_{\mathrm{rul}}(\theta)
+
\lambda_{\mathrm{cls}}\mathcal L_{\mathrm{cls}}(\theta).
\label{eq:total_loss}
\end{equation}
All reported experiments use
$\lambda_{\mathrm{rul}}=1$ and
$\lambda_{\mathrm{cls}}=5000$.
An optional cross-entropy branch exists in the implementation but is
disabled in the experiments reported here.

The training loop is summarized in Algorithm~\ref{alg:weighted_tdn}
(Appendix~\ref{app:weighted_tdn}): store transition subsequences, form the vector
TD$(n,\lambda)$ target from
Eqs.~\eqref{eq:vector_tdn_target}--\eqref{eq:vector_tdnl_target}, and
update the shared network by a semi-gradient step on
Equation~\eqref{eq:total_loss}.

\subsection{Implementation Details}
\label{sec:implementation_details}

\subsubsection{Tuning Parameter Selection}
\label{sec:tuning_parameter_selection}

The most important GVF-specific tuning parameter is the continuation $\gamma_{t+1}$. In this paper, $\gamma_{t+1}$ is not only an algorithmic discount factor; it is part of the GVF specification. It multiplies the future value term in the Bellman target and therefore determines how much future cumulants matter. Setting $\gamma_{t+1}=0$ stops accumulation after the next transition, setting $\gamma_{t+1}=1$ continues accumulation until terminal absorption, and using a constant $\gamma_{t+1}=\bar\gamma<1$ creates a geometrically discounted prediction.

For undiscounted cycle-count RUL, the natural continuation is $\gamma_{t+1}^{\mathrm{rul}}=\one\{S_{t+1}\notin\cF\}$, which asks for expected time to absorption. If a practitioner instead wants a finite planning window, choose $\bar\gamma<1$ so that the soft planning scale $(1-\bar\gamma)^{-1}$ matches a desired horizon of $H$ operating cycles (e.g., $\bar\gamma=1-1/H$). This is a \emph{soft} horizon: outcomes beyond $H$ cycles still enter the return, but their contribution decays geometrically. Changing the cumulant itself would define a different GVF; in this paper we keep the survival-time cumulant and control temporal scale through $(\bar\gamma,n,\lambda)$.

The remaining tuning parameters control the estimator rather than the meaning of the GVF. The truncation length $n$ and trace parameter $\lambda$ should be tuned jointly with $\bar\gamma$ and reported together, because they determine the effective temporal scale of the learned prediction.

\subsubsection{Function Approximation}
\label{sec:function_approximation}

Let $V_\theta(s)\in\R^m$ be a parametric vector-GVF approximator. In the multi-task
setting, the network has a shared encoder $f_\theta(s)$ and task-specific heads:
\begin{equation}
\widehat V_\theta(s)
=
\left(
\widehat V_\theta^{\mathrm{time}}(s),
\widehat H_\theta^{(1)}(s),
\ldots,
\widehat H_\theta^{(K)}(s)
\right).
\end{equation}
In all reported experiments the mode head outputs unconstrained scores
$\widehat H_\theta(s)\in\R^K$ and is trained by the mode block of the weighted
vector-GVF MSE loss (Eq.~\eqref{eq:total_loss}), not by cross-entropy.
When the mode components target $\gamma_{\mathrm{mode}}=1$, a $K$-class softmax
readout of $\widehat H_\theta(s)$ may still be used at evaluation to estimate
$\Prob_s(Z=k)$. When the mode components target finite $\gamma_{\mathrm{mode}}<1$,
the raw mode GVFs $H_{\gamma_{\mathrm{mode}}}^{(k)}(s)$ are discounted hitting
values and need not sum to one; a softmax readout should then be interpreted as
estimating normalized urgency-weighted mode scores
(Proposition~\ref{prop:terminal_event_mode_tilting}).

\subsubsection{Predictive State Representation}
\label{sec:state_representation}
\label{sec:windowed_states}
The Bellman recursions are defined with respect to the predictive state $S_t$. In practice, the physical degradation state is unobserved, so $S_t$ must summarize the sensor history relevant to RUL and failure-mode prediction. A suitable representation may be a fixed sensor window, a recurrent hidden state, or another learned history embedding. In the experiments, we use the fixed-length window
\begin{equation}
S_t=[X_{t-h+1},\ldots,X_t],
\end{equation}
where $X_t$ is the multivariate sensor observation and $h$ is the window length. The same state representation is shared by the RUL and failure-mode GVF heads. For C-MAPSS and the simulation study, a one-dimensional convolutional encoder extracts temporal features from this window before the task-specific outputs are computed. A finite window generally makes the Markov property approximate rather than exact: trajectories with identical recent observations may differ in earlier history that remains prognostically relevant. The Bellman equations should therefore be interpreted relative to the chosen predictive state. The GVF formulation is otherwise architecture-agnostic: the distinction from conventional supervised prognostics lies in the Bellman-consistent target construction and terminal semantics, not in the use of a particular neural encoder.

\subsection{Statistical Properties of the Estimator}
\label{sec:algorithmic_properties}

Proofs appear in Appendix~\ref{app:proof_algorithmic_properties}.

\subsubsection{Properties of the Two Prognostic GVF Types}
\label{sec:gvf_readouts}

\begin{proposition}[Soft-horizon bound]
\label{prop:soft_horizon_bound_theory}
For any $s\in\cC$ and $\gamma\in[0,1)$,
$0\le V_\gamma^{\mathrm{time}}(s)\le (1-\gamma)^{-1}$.
More generally, any GVF component with bounded cumulant
$|u_{t+1}^{(i)}|\le\bar c_i$ and continuation
$\gamma_{t+1}^{(i)}\le\bar\gamma_i<1$ satisfies
$\|V_i^\star\|_\infty\le \bar c_i/(1-\bar\gamma_i)$
(Appendix~\ref{app:proof_soft_horizon_bound}).
\end{proposition}
Thus the discounted survival-time GVF is automatically bounded, with effective planning
horizon approximately $(1-\gamma)^{-1}$. Therefore, $\gamma$ can be selected to match a
desired prediction horizon in the range $(1-\gamma)^{-1}$.

\begin{proposition}[Finite-$\gamma$ mode tilting]
\label{prop:terminal_event_mode_tilting}
Assume absorption occurs almost surely from every $s\in\cC$, i.e.\
$\Prob_s(\tau_t<\infty)=1$. For $s\in\cC$ and $\gamma\in[0,1)$, let $p_k(s)=\Prob_s(Z=k)$
and $a_k(s;\gamma)=\E_s[\gamma^{\tau_t-1}\mid Z=k]$. The mode GVF of
Definition~\ref{def:mode_terminal_event_gvf} factorizes as
\begin{equation}
H_\gamma^{(k)}(s)=p_k(s)a_k(s;\gamma).
\label{eq:mode_tilting_factorization}
\end{equation}
Under almost-sure absorption $\sum_{j=1}^K p_j(s)=1$, so as $\gamma\uparrow1$ we have
$H_\gamma^{(k)}(s)\to p_k(s)$ and the normalized score
$\tilde p_{\gamma,k}(s)=H_\gamma^{(k)}(s)/\sum_{j=1}^K H_\gamma^{(j)}(s)\to p_k(s)$.
Without almost-sure absorption the normalized score instead converges to the
conditional probability $p_k(s)/\sum_{j=1}^K p_j(s)=\Prob_s(Z=k\mid \tau_t<\infty)$.
\end{proposition}
Thus the normalized score $\tilde p_{\gamma,k}(s)$ is a reweighted terminal-mode
score: the factor $a_k(s;\gamma)=\E_s[\gamma^{\tau_t-1}\mid Z=k]$ tilts mass toward
modes with shorter expected time to failure, producing a finite-$\gamma$,
urgency-weighted mode score rather than a calibrated hitting probability. In the
$\gamma\uparrow1$ limit it reduces to the terminal-mode probability $p_k(s)$ (to
the conditional probability given absorption when absorption is not almost sure);
for finite $\gamma$, the residual tilt depends on the mode-dependent hitting-time
distribution and is not guaranteed to be negligible for long or heavy-tailed times
to failure.

\subsubsection{Vector Bellman Fixed Point and TD($n,\lambda$) Operator}
\label{sec:tdnl_operator_properties}

Let $\mathcal{T}$ denote the vector Bellman operator in Eq.~\eqref{eq:vector_bellman}.
For $k\ge1$, let $\mathcal{T}^k$ denote the $k$-fold Bellman operator. The population
operator corresponding to the truncated TD($n,\lambda$) target is
\begin{equation}
\mathcal{T}_{n,\lambda}
=
(1-\lambda)\sum_{k=1}^{n-1}\lambda^{k-1}\mathcal{T}^k
+
\lambda^{n-1}\mathcal{T}^n.
\label{eq:tdnl_operator}
\end{equation}

\begin{theorem}[Bellman target and linear TD limit]
\label{thm:bellman_td_convergence}
Under the standard on-policy linear-TD assumptions detailed in
Appendix~\ref{app:proof_algorithmic_properties}, fix a GVF component
$i$ with continuation $\gamma_{t+1}^{(i)}\le\bar\gamma_i<1$.
Write
\[
(\mathcal T_i v)(s)
=
\E\!\bigl[
u_{t+1}^{(i)}+\gamma_{t+1}^{(i)}v(S_{t+1})
\mid S_t=s
\bigr]
\]
for the component Bellman operator induced by
Eq.~\eqref{eq:vector_bellman}, and let
$\mathcal T_{i,n,\lambda}$ denote the corresponding truncated
TD$(n,\lambda)$ operator induced componentwise by
Eq.~\eqref{eq:tdnl_operator}.

\begin{enumerate}[label=(\roman*),leftmargin=*]
\item The operators $\mathcal T_i$ and $\mathcal T_{i,n,\lambda}$ are
contractions---with modulus at most
\[
\rho_{i,n,\lambda}
=
(1-\lambda)\sum_{k=1}^{n-1}\lambda^{k-1}\bar\gamma_i^k
+\lambda^{n-1}\bar\gamma_i^n
<1
\]
for $\mathcal T_{i,n,\lambda}$---and share the same unique fixed point
$V_i^\star$.

\item For a linear approximation $v_{w_i}(s)=\phi(s)^\top w_i$ and
$L_2(\Psi)$ projection $\Pi_\Psi$ onto the feature span, on-policy
semi-gradient TD$(n,\lambda)$ converges almost surely to the unique
projected fixed point
\[
v_i^{\mathrm{TD}}
=
\Pi_\Psi\mathcal T_{i,n,\lambda}v_i^{\mathrm{TD}},
\]
or equivalently $w_{i,t}\to w_i^\star=A_{i,n,\lambda}^{-1}b_{i,n,\lambda}$
\citep{tsitsiklis1997analysis}.

\item Complete-return Monte Carlo regression converges to the
least-squares projection $v_i^{\mathrm{MC}}=\Pi_\Psi V_i^\star$.
Hence, if $V_i^\star$ lies in the feature span (realizability),
\[
v_i^{\mathrm{TD}}=v_i^{\mathrm{MC}}=V_i^\star;
\]
under function-approximation misspecification the two limits may
differ: MC projects $V_i^\star$ directly, whereas TD solves the
projected Bellman equation
\citep{tsitsiklis1997analysis}.
\end{enumerate}

Because $\Gamma_{t+1}$ is diagonal, these conclusions extend
componentwise to the stacked vector GVF.
\end{theorem}

The assumptions, contraction modulus, projected normal equations, and
proofs are given in Appendix~\ref{app:proof_algorithmic_properties}.

\subsubsection{MC versus TD Target Variance}
\label{sec:mc_td_target_variance}

MC uses the complete realized return to absorption (Eq.~\eqref{eq:vector_mc_target}); it does \emph{not} bootstrap.
TD($n$) uses the same first $n$ observed cumulants, then replaces the unobserved remainder
by a successor-state prediction (Eq.~\eqref{eq:vector_tdn_target}).
In censored or partially observed prognostics, the full MC return is often unavailable
even when local transitions are observed.
The following proposition isolates the variance effect of replacing the residual MC
return by an oracle $n$-step TD bootstrap $V^\star$.

\begin{proposition}[MC--TD target-variance decomposition]
\label{prop:mc_td_target_variance}
Fix a scalar GVF with continuation $\gamma\in[0,1)$, fix $n\ge1$, and condition throughout on
the event $\{S_t=s,\ \tau_t>n\}$ (i.e., absorption has not yet occurred by step $t+n$; this
event has positive probability whenever $\Prob_s(\tau_t>n)>0$).
With prefix $P_{t,n}=\sum_{\ell=0}^{n-1}\gamma^{\ell}c_{t+\ell+1}$ and residual $R_{t+n}=\sum_{\ell=0}^{\tau_{t+n}-1}\gamma^{\ell}c_{t+n+\ell+1}$,
\begin{equation}
G_t^{\mathrm{MC}}=\sum_{\ell=0}^{\tau_t-1}\gamma^{\ell}c_{t+\ell+1}=P_{t,n}+\gamma^{n}R_{t+n}
\quad\text{on }\{\tau_t>n\}.
\label{eq:mc_full_return}
\end{equation}
Under Markovity of $S_t$, the oracle TD target $G_t^{\mathrm{TD},\star}=P_{t,n}+\gamma^{n}V^\star(S_{t+n})$ with $V^\star(S_{t+n})=\E[R_{t+n}\mid S_{t+n}]$ equals the conditional expectation of $G_t^{\mathrm{MC}}$ given the observed history through $S_{t+n}$, and
\begin{equation}
\Var(G_t^{\mathrm{MC}}\mid S_t=s,\,\tau_t>n)
=
\Var(G_t^{\mathrm{TD},\star}\mid S_t=s,\,\tau_t>n)
+
\gamma^{2n}\E[\Var(R_{t+n}\mid S_{t+n})\mid S_t=s,\,\tau_t>n].
\label{eq:mc_td_variance_decomp}
\end{equation}
Thus, on the event that absorption has not yet occurred by step $t+n$, oracle TD has no
larger conditional variance than complete MC.
\end{proposition}

Both targets share $P_{t,n}$; MC keeps the realized residual $R_{t+n}$, while oracle TD replaces it by $\E[R_{t+n}\mid S_{t+n}]$. With a learned bootstrap $\widehat V$, TD helps when successor prediction is more accurate than sampling the full residual.

\section{Simulation Study}
\label{sec:sim_study}

The central hypothesis is that Bellman-consistent targets exploit transition-level
information and remain advantageous as terminal labels become scarce.

\subsection{Data Generation}
\label{sec:exp_sim_dgp}

Each unit follows a two-stage degradation process with a shared early phase and a
mode-specific late phase. An anomaly trigger ends the early phase and starts the
late-phase clock; monitoring for prognostics begins only after that trigger, so
all RUL and failure-mode targets are defined on the post-trigger trajectory.
The two modes differ in both the late-phase degradation kernel and a post-trigger
discriminative sensor pattern, as specified next.

\begin{align}
&\quad g_A(t)=\gamma^A_0+\gamma^A_1(e^t\!-\!1)+\gamma^A_2 t^3,
  \quad \boldsymbol{\gamma}^A_i\sim\mathcal{N}(\boldsymbol{\mu}_A,\boldsymbol{\Sigma}_A),
  \quad \boldsymbol{\mu}_A=(0,0.5,1.5)^\top, \notag\\
&\quad t_{\mathrm{act},i}=\min\{t:g_A(t)\ge\theta_i\},
  \quad \theta_i\sim\mathrm{U}(0.4,0.6),
  \quad s=t-t_{\mathrm{act},i}, \notag\\
&\quad g_B(s)=\begin{cases}
    \gamma^B_0+\gamma^B_1 s+\gamma^B_2 s^2, \ \ \boldsymbol{\gamma}^B\sim t_5(\boldsymbol{\mu}_B,\boldsymbol{\Sigma}_B) & \text{(Mode~1)}\\
    \gamma^B_0+\gamma^B_1(e^s\!-\!1)+\gamma^B_2 s^3, \ \ \boldsymbol{\gamma}^B\sim\mathcal{N}(\boldsymbol{\mu}_A,\boldsymbol{\Sigma}_A) & \text{(Mode~2)}
  \end{cases},
  \quad \gamma^B_1,\gamma^B_2>0, \label{eq:dgp}\\[4pt]
&\quad \boldsymbol{x}(t)=\boldsymbol{x}_A(t,g_A)+\rho_B\,\boldsymbol{x}_B(s,g_B)+\boldsymbol{p}(s),
  \quad \rho_B=1.8,
  \quad \boldsymbol{\omega}^\top\!\boldsymbol{x}_{\cdot}=g_{\cdot},
  \quad \boldsymbol{\omega}=(0.6,0.2,-0.5)^\top, \notag\\
&\quad \begin{pmatrix}x_1\\x_2\\x_4\end{pmatrix}
  =\begin{pmatrix}a_{11}t^2-a_{12}\sin(25t)\\
                  a_{21}t+a_{22}\sin(50t)\\
                  a_{41}t+a_{42}\end{pmatrix}+\boldsymbol{\varepsilon},
  \quad x_3=(g-\omega_1 x_1-\omega_2 x_2)/\omega_3,
  \quad \varepsilon_j\sim\mathcal{N}(0,0.05^2), \notag\\
&\quad p_0(u)=\begin{cases}
    \alpha\, u^{1.5} & \text{Mode~1}\\
    \alpha\sin(n\cdot 2\pi u)\sqrt{u}+\alpha\cdot 0.15\,u & \text{Mode~2}
  \end{cases},
  \quad \boldsymbol{p}(s)=p_0(u)\,(1,0,0,0)^\top,
  \quad \alpha\sim\mathrm{U}(0.10,0.25),\quad n\sim\mathrm{U}(15,25),
  \quad u=s/s_{\mathrm{fail}}, \notag
\end{align}

\noindent where failure occurs when $g_B(s)\ge 5$ and $\boldsymbol{\mu}_B=(0.2,1.0,2.0)^\top$.
The pre-trigger channels also receive mode-independent sinusoidal texture
(two-frequency oscillation plus linear drift on $x_1,x_2$), ensuring pre-trigger
sensor distributions are statistically identical across modes.
Beyond the mode-specific late-phase kernel $g_B$, the enrichment
$\boldsymbol{p}(s)$ is the sole additional mode-specific sensor pattern:
it loads $p_0(u)$ onto channel $x_1$ only and vanishes at the trigger ($s=0$);
its amplitude grows through Phase~B (monotonically for Mode~1, and as a
$\sqrt{u}$ envelope modulating a sinusoid for Mode~2, so only the envelope grows)
--- so macro-F1 starts near chance and rises toward end of life, consistent with
Table~\ref{tab:sim_main}.

Trajectories are sampled at $\Delta t=0.005$ with $\mathrm{max\_time\_B}=10.0$,
yielding ${\approx}300$ post-trigger steps per unit.
We generate $300$ historical units per mode ($600$ pooled), split $70/30$ into training
and validation, and $300$ in-service test units per mode ($600$ pooled), with a fixed
seed shared by all methods.  The RUL target at step $t$ is
$r_t = \lfloor(\tau_i - t)/\Delta t\rfloor$ remaining steps, post-trigger only.
The mode label is the integer $\ell_i \in \{0,1\}$. Figure~\ref{fig:et_long_dataset_viz} shows representative trajectories for both modes.

\begin{figure}[htbp]
  \centering
  \includegraphics[width=\linewidth]{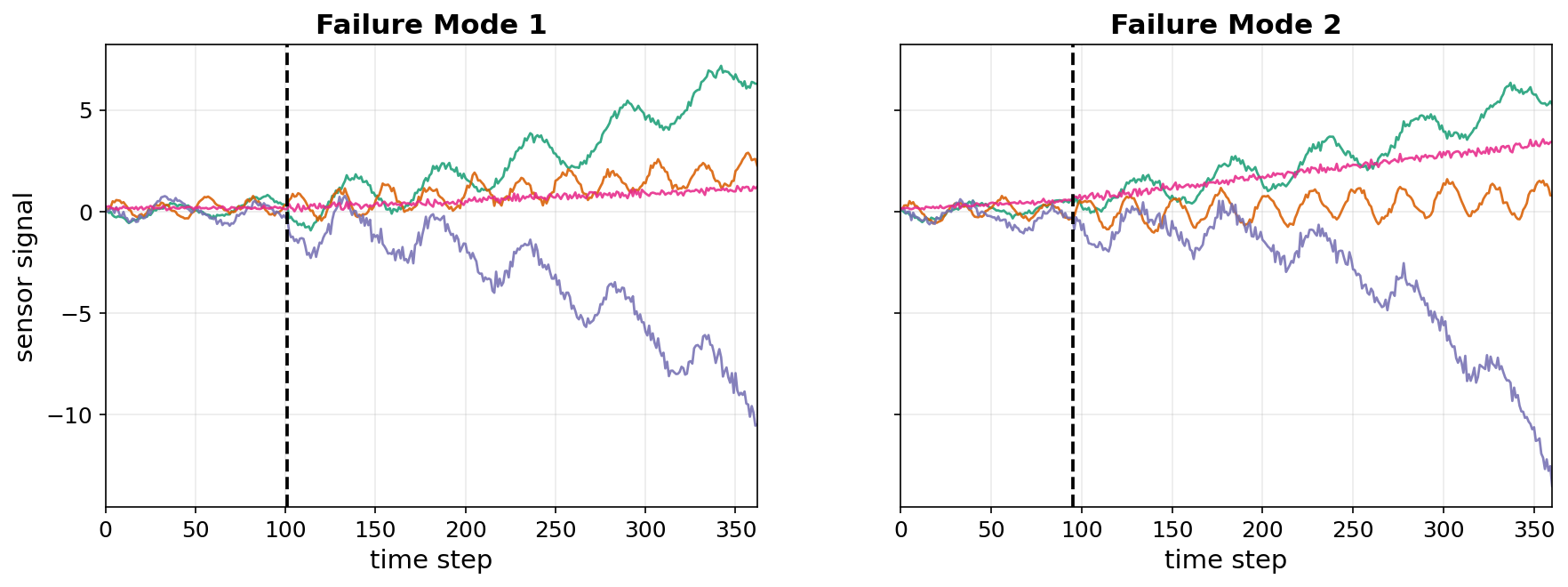}
  \caption{Simulation: example trajectories for both modes. Modes differ in both
    the late-phase kernel $g_B$ and the additional $x_1$ pattern $p_0$; the latter
    vanishes at the trigger, so macro-F1 starts near chance and rises toward failure.}
  \label{fig:et_long_dataset_viz}
\end{figure}

\subsection{Evaluation Protocol}
\label{sec:exp_protocol}

All reported results use the same train/validation/test partitions within each platform. The full-data experiments compare methods when every training trajectory is observed to failure; the label-scarce experiments remove terminal supervision from a fraction of training segments.

\paragraph{Benchmark Methods}
To ensure a fair comparison, TD$(n,\lambda)$, the supervised same-backbone Monte Carlo (MC) control, and the Health Index (HI) baseline all use \emph{the same one-dimensional convolutional sliding-window encoder}. For MC, we jointly trained a standard regression head for RUL prediction and classification head for failure-mode prediction using the same CNN1D backbone. For HI, we trained a scalar health-index model using the same CNN1D backbone and a learned health-index output and a terminal-threshold variance penalty to encourage monotonicity; implementation details are given in Appendix~\ref{app:hi_baseline}.

Replacing the shared CNN1D with a single-step DNN degraded TD, so all reported models use the windowed CNN1D backbone. Inputs are $(W \times F)$ sliding windows with $F{=}4$, permuted channel-first. The encoder uses Conv1D channels $F \rightarrow 256 \rightarrow 96 \rightarrow 32$, kernel size $7$, ReLU, global average pooling, and FC layers $32 \rightarrow 64 \rightarrow 128$. The trunk branches to an RUL head (FC $128 \rightarrow 1$) and a failure-mode head (FC $128 \rightarrow C$, $C{=}2$). GVF estimators train both heads jointly by the weighted vector-GVF MSE loss in Eq.~\eqref{eq:total_loss} with $n{=}5$, $\bar\gamma{=}0.995$ for the RUL continuation, mode continuation $\gamma_{\mathrm{mode}}{=}1$ (undiscounted terminal-event targets), $\lambda{=}1.0$ on full trajectories ($0.7$ for stitch), $\lambda_{\mathrm{rul}}{=}1$, $\lambda_{\mathrm{cls}}{=}5000$, $30$ epochs, and validation-based model selection. MC regresses true cycle-count RUL and terminal failure-mode labels with \emph{the same CNN1D backbone and joint classification head}; it is a supervised same-backbone control, not a survey of conventional RUL methods.

\paragraph{Evaluation Metrics}
The TD RUL head learns $V_\gamma^{\mathrm{time}}$ at $\bar\gamma{=}0.995$ (Section~\ref{sec:gvf_readouts}). We convert raw output $\hat v$ to cycle-count RUL by $\hat r = \log\bigl(1-(1-\bar\gamma)\hat v\bigr)/\log\bar\gamma$, the deterministic-hitting-time inversion of $\hat v=(1-\bar\gamma^{\hat r})/(1-\bar\gamma)$. Under stochastic hitting times, this is a certainty-equivalent readout, not generally $\E[T\mid s]$. We clip $1-(1-\bar\gamma)\hat v$ to $[10^{-12},1.0]$ and the final RUL to be nonnegative. MC regresses cycle-count RUL directly.
Rather than scoring a single end-of-observation prediction per unit, we score every post-trigger step: for each test unit the model predicts RUL and failure mode at all time steps from the first full window to failure, yielding approximately $144{,}000$ predictions over the $600$ test units.
For RUL we report the \emph{normalized RUL error} (RUL NAE), which divides each per-step absolute error by the unit's post-trigger lifetime $T_B^{(i)}$, where $\hat{r}_{it}$ and $r_{it}$ are the predicted and true RUL of unit $i$ at post-trigger step $t$:
\begin{equation}
\mathrm{NAE}_{it} = \frac{\bigl|\hat{r}_{it}-r_{it}\bigr|}{T_B^{(i)}};
\label{eq:metric_nae}
\end{equation}
this per-step quantity is averaged within each unit--bucket cell to form a
fleet-level bucket mean for a given training run; Table~\ref{tab:sim_main}
then reports mean~$\pm$~SEM of those fleet means across three training seeds.
Buckets partition the remaining-life percentage $100r_{it}/T_B^{(i)}$ into five equal intervals.
The same by-bucket NAE construction (without the across-seed aggregation stated in
each caption) underlies Figure~\ref{fig:cmapss_rul_boxplot}.
Table~\ref{tab:cmapss_fulltraj} instead reports the per-test-unit overall mean of
$\mathrm{NAE}_{it}$---i.e.\ $\mathrm{NAE}_i$ averaged within each engine over all its
observed steps---as mean~$\pm$~SEM over test engines (RUL columns); endpoint F1 on
FD003/FD004 uses mean~$\pm$~SEM across training seeds (see that table's caption).
For failure-mode prediction we report accuracy and macro-F1.

Unless a caption states otherwise, results are reported under the following seed protocol: full-trajectory simulation results use $3$ training seeds; reported simulation stitch results use $\lambda{=}0.7$.

\paragraph{Full-trajectory and stitch experimental protocols.}

Let a full trajectory be
\[
T_i=(X_{i,0},\ldots,X_{i,T_i})
\]
with unit identity $i$.
In the \emph{full-trajectory} setting, each training unit is observed run-to-failure. The training pipeline retains unit identity, absolute cycle index, complete trajectory length, remaining life, and terminal mode for target construction. The network input itself remains the corresponding sensor window. This standard supervised setup supports both complete-return MC targets $G_t^{\mathrm{MC}}$ and bootstrapped TD targets.

In the \emph{label-scarce stitch} setting, training draws anonymous fixed-length segments from the full-trajectory pool. A segment is $S_j = (X_{a_j:b_j}, X_{a_j+1:b_j+1}, \ldots, X_{a_j+W-1:b_j+W-1})$ of length $W+1$, sampled uniformly from any $(i,t)$. The protocol withholds unit identity and complete-return labels from nonterminal segments, preserving only adjacent-window pairs $(S_t, S_{t+1})$, immediate cumulants $u_{t+1}$, and continuations $\Gamma_{t+1}$ (terminal if observed, otherwise capped at $W+1$). Models do not observe $i$, absolute cycle time $t$, $T_i$, $\tau_t$, or $Z_i$ unless the segment reaches failure within $W+1$ steps. Complete-return MC targets are unavailable for nonterminal segments; bootstrapped TD targets remain defined. The protocol models fragmented records, not arbitrary informative censoring.

\subsection{Results}
\label{sec:exp_sim_results}

We stress-test data efficiency by training on available fractions $f\in\{0.2,0.5,0.8\}$ of the historical segment pool.

\subsubsection{Full-Trajectory Setting}
\label{sec:exp_sim_fulltraj}

Table~\ref{tab:sim_main} summarizes full-trajectory joint results
(mean~$\pm$~SEM across three training seeds; see caption).

\begin{table}[htbp]
\centering
\caption{Simulation: RUL NAE ($\downarrow$) and macro-F1 ($\uparrow$) by remaining-life bucket.
For each training seed, one fleet-level mean is computed per cell on the fixed test cohort;
cells report mean~$\pm$~SEM across $N_{\rm seeds}{=}3$ seeds
($\mathrm{SEM}=s/\sqrt{N_{\rm seeds}}$).
\textbf{Bold} = best mean; TD uses $\lambda{=}1.0$; HI has no classification head.
Bucket labels are post-trigger remaining-life percentages
$100r_{it}/T_B^{(i)}$ (Section~\ref{sec:exp_protocol}).
TD denotes TD$(n,\lambda$).}
\label{tab:sim_main}
\setlength{\tabcolsep}{3.5pt}
\begin{tabular}{@{}lccc@{}}
\toprule
\multicolumn{4}{@{}c@{}}{\emph{RUL NAE $\downarrow$}} \\
Bucket & TD & MC & HI \\
\midrule
80--100 & $\mathbf{0.107\pm0.004}$ & $0.113\pm0.003$ & $0.223\pm0.061$ \\
60--80  & $\mathbf{0.090\pm0.003}$ & $0.100\pm0.003$ & $0.188\pm0.039$ \\
40--60  & $\mathbf{0.069\pm0.001}$ & $0.080\pm0.006$ & $0.149\pm0.022$ \\
20--40  & $\mathbf{0.043\pm0.002}$ & $0.055\pm0.004$ & $0.106\pm0.011$ \\
0--20   & $\mathbf{0.024\pm0.001}$ & $0.033\pm0.002$ & $0.068\pm0.002$ \\
Overall & $\mathbf{0.0615\pm0.0003}$ & $0.072\pm0.003$ & $0.137\pm0.022$ \\
\bottomrule
\end{tabular}
\hfill
\begin{tabular}{@{}lcc@{}}
\toprule
\multicolumn{3}{@{}c@{}}{\emph{Macro-F1 $\uparrow$}} \\
Bucket & TD & MC \\
\midrule
80--100 & $\mathbf{0.699\pm0.006}$ & $0.596\pm0.061$ \\
60--80  & $\mathbf{0.822\pm0.016}$ & $0.728\pm0.040$ \\
40--60  & $\mathbf{0.836\pm0.003}$ & $0.794\pm0.022$ \\
20--40  & $\mathbf{0.889\pm0.007}$ & $0.850\pm0.025$ \\
0--20   & $\mathbf{0.933\pm0.012}$ & $0.887\pm0.024$ \\
Overall & $\mathbf{0.852\pm0.008}$ & $0.792\pm0.031$ \\
\bottomrule
\end{tabular}
\end{table}

Under complete run-to-failure labels, TD and MC have comparable overall RUL accuracy, so GVF prediction loses no accuracy relative to the supervised same-backbone control. TD's advantage is strongest near failure (0--20\% bucket) and TD leads MC in macro-F1 at every life stage, with the largest gap early in life. HI trails both GVF methods and is more seed-sensitive.
\subsubsection{Label-Scarce Stitch Setting}
\label{sec:exp_et_stitch}

Figure~\ref{fig:et_stitch_frac_boxplot} shows RUL NAE at $f\in\{0.2, 0.5, 0.8\}$. Classification is omitted: TD and MC are near chance at these fractions, and Table~\ref{tab:sim_main} reports full-trajectory classification only.

\begin{figure}[htbp]
  \centering
  \includegraphics[width=\linewidth]{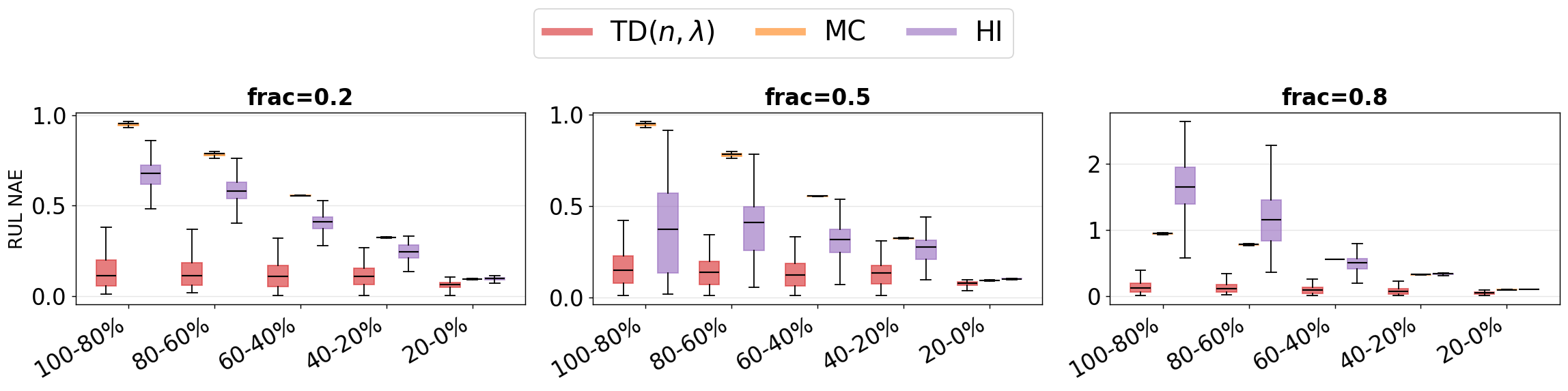}
  \caption{Simulation: label-scarce stitch RUL NAE ($\downarrow$) at fractions $f\in\{0.2, 0.5, 0.8\}$ for TD$(n,\lambda)$ ($\lambda{=}0.7$), MC, and HI (Section~\ref{sec:exp_et_stitch}).}
  \label{fig:et_stitch_frac_boxplot}
\end{figure}

Under the stitch protocol of Section~\ref{sec:exp_protocol}, TD remains competitive across fractions, while MC and HI underperform: MC is limited to terminal-containing segments, and HI's first-passage calibration is compromised under anonymity (stitch RUL error non-monotonic across remaining life, worst at mid-life).

\subsection{Ablation Studies}
\label{sec:exp_sim_lambda}

Figure~\ref{fig:td_mc_joint_paper} shows how the trace parameter $\lambda$ shapes the
temporal coherence of predicted RUL and failure-mode-risk trajectories.  Small $\lambda$
yields local bootstrapping (TD(1)-like targets), producing noisier per-step predictions;
large $\lambda$ blends in longer observed returns and produces smoother trajectories.
The figure uses $\lambda{=}0.5$ to maximize the visual contrast between TD and MC; the
main table uses the canonical $\lambda{=}1.0$ (full-trajectory) and $\lambda{=}0.7$
(stitch), which were selected on validation data.

\begin{figure}[htbp]
  \centering
  \includegraphics[width=0.8\linewidth]{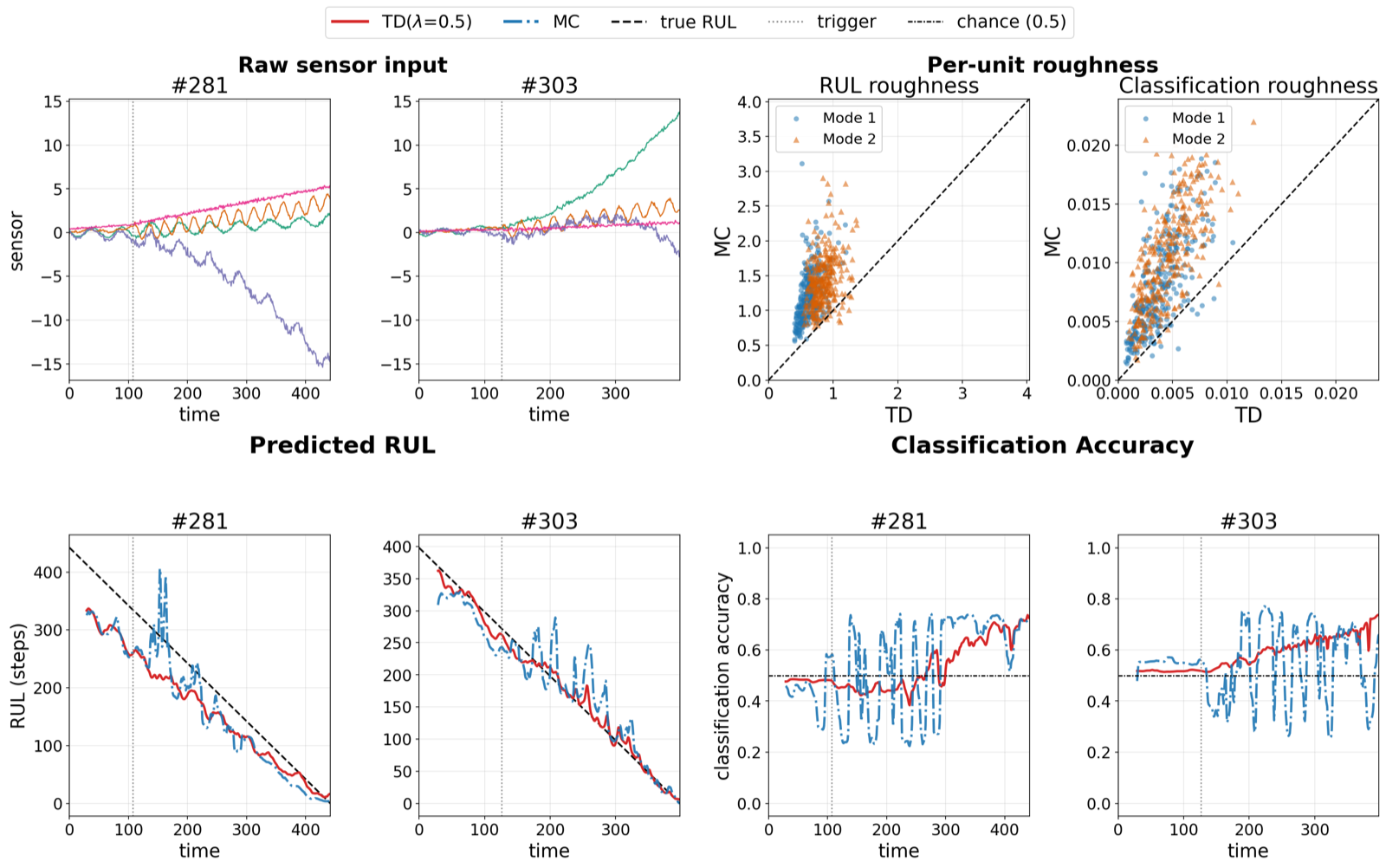}
  \caption{Simulation: $\lambda$ effect on prediction smoothness ($\lambda{=}0.5$
    shown for contrast; canonical full-trajectory model uses $\lambda{=}1.0$).
    TD$(n,\lambda)$ gives smoother RUL and failure-risk trajectories than MC.}
  \label{fig:td_mc_joint_paper}
\end{figure}

\section{Case Study: C-MAPSS}
\label{sec:exp_cmapss}

The simulation study isolates Bellman consistency under known dynamics; C-MAPSS tests transfer when the dynamics are unknown and Section~\ref{sec:problem_formulation}'s assumptions hold only approximately. The NASA C-MAPSS turbofan benchmark \citep{saxena2008damage} provides physics-simulated multivariate trajectories across four subsets varying operating-condition heterogeneity and competing failure modes, with run-to-failure labels for evaluation. FD003/FD004 failure-mode labels are not native to C-MAPSS; following \citet{fu2025degradation}, we derive reference pseudo-labels by DTW $k$-means ($k{=}2$) on held-out sensor-trajectory embeddings, with no RUL or endpoint information. These labels score failure-mode F1 and provide MC's supervised mode target; TD receives only the supervision specified for each protocol. C-MAPSS therefore localizes performance differences to known difficulty axes.

\subsection{Dataset Overview}
\label{sec:exp_cmapss_dataset}

C-MAPSS simulates turbofan engine degradation under user-specified operating conditions and
fault modes, producing multivariate sensor trajectories from a nominal state to failure.
Each of the four subsets isolates a different combination of two difficulty axes:
FD001 provides a single-condition/single-fault sanity check; FD002 adds operating-condition
heterogeneity; FD003 adds a second, competing failure mode under the same single condition;
FD004 combines both operating-condition heterogeneity and dual failure modes, and is
therefore the most difficult subset in the suite.
Train/test sizes, operating conditions, and fault modes are FD001 $100/100$, $1$, $1$;
FD002 $260/259$, $6$, $1$; FD003 $100/100$, $1$, $2$; and FD004 $249/248$, $6$, $2$.

\subsection{Data Preprocessing}
\label{sec:exp_cmapss_preproc}

Training units are run to failure; test units are truncated before failure with residual
life withheld until evaluation.
We retain 15 informative sensors (dropping constant/near-constant channels and raw
operating-setting columns), MinMax-scale by operating condition on FD002/FD004 and
globally on FD001/FD003 (training statistics only), and form sliding windows of length
$W{=}30$ (FD001/FD003/FD004) or $W{=}20$ (FD002; shortened following
\citet{Li2018remaining}).
Prefixes shorter than $W$ are left-zero-padded.
Unless noted otherwise, C-MAPSS GVF runs use the same CNN1D encoder as
Section~\ref{sec:exp_protocol}, with $n{=}4$ (stitch padding $n{-}1$), shared
soft-horizon continuation $\bar\gamma{=}0.995$ for both the RUL and mode heads
(so $\gamma_{\mathrm{mode}}{=}0.995$, unlike the simulation's undiscounted
$\gamma_{\mathrm{mode}}{=}1$), $60$ training epochs ($80$ for stitch), and
weighted vector-GVF MSE loss weights $\lambda_{\mathrm{rul}}{=}1$,
$\lambda_{\mathrm{cls}}{=}5000$ on the multi-mode subsets FD003/FD004
(FD001/FD002 are trained RUL-only).
Each training run uses a \emph{single} trace parameter $\lambda$ for the shared
vector TD$(n,\lambda)$ target of both heads
(Eqs.~\eqref{eq:vector_tdnl_target}--\eqref{eq:total_loss}).
Full-trajectory TD is trained as a $\lambda$ sweep of separately trained joint
models; stitch $\lambda$ is chosen by the cross-configuration rule below.
Within each run, we use best-validation-epoch weights on held-out training
trajectories; test fleets stay disjoint.
Across hyperparameter configurations, the reported configuration minimizes
test-endpoint RMSE, following C-MAPSS practice because no disjoint split exists
for that comparison.
Endpoint RMSE selection need not optimize reported NAE, but the rule is uniform
across methods and subsets.
RUL is evaluated at every post-window step under the NAE protocol of
Section~\ref{sec:exp_protocol}; all methods share this pipeline.

\subsection{Full-Trajectory Results}
\label{sec:exp_cmapss_results}
\label{sec:exp_cmapss_fulltraj}

Table~\ref{tab:cmapss_fulltraj} and Figure~\ref{fig:cmapss_rul_boxplot} summarize
full-trajectory results across all four C-MAPSS subsets.
RUL NAE is the per-test-unit mean of $|\hat{r}-r|/T_B$, reported as
mean~$\pm$~SEM over test engines; for FD003/FD004, the rightmost columns give endpoint minority-class F1 (mean~$\pm$~SEM across three training seeds).
Reported full-trajectory TD numbers use the two $\lambda$ settings from that
sweep of separately trained joint models.
RUL uses the best-endpoint-RMSE configuration ($\lambda{=}1.0$).
Failure-mode classification on FD003/FD004 uses a separately trained
$\lambda{=}0$ model (best endpoint RMSE among $\lambda{=}0$ runs), because
$\lambda{=}1.0$ harms the shared representation's mode head under this
benchmark.
By-bucket classification accuracy is deferred to Appendix~\ref{app:cmapss_cls_bucket}.
\textbf{Bold} marks the best RUL NAE mean; FD003/FD004 F1 are unbolded (within noise of each other).

\begin{table}[htbp]
\centering
\caption{C-MAPSS full-trajectory RUL NAE ($\downarrow$; Section~\ref{sec:exp_protocol})
and endpoint minority-class F1 on FD003/FD004 ($\uparrow$).
RUL NAE: mean~$\pm$~SEM over test engines (one per-engine mean NAE, then SEM across engines).
F1: mean~$\pm$~SEM across $3$ training seeds (one endpoint F1 per seed).
TD denotes TD$(n,\lambda$). \textbf{Bold} = best RUL NAE mean per dataset.}
\label{tab:cmapss_fulltraj}
\begin{tabular}{@{}lccccc@{}}
\toprule
& \multicolumn{3}{c}{RUL NAE $\downarrow$} & \multicolumn{2}{c}{F1 $\uparrow$} \\
\cmidrule(lr){2-4}\cmidrule(lr){5-6}
Dataset & TD & MC & HI & TD & MC \\
\midrule
FD001 & $0.224\pm0.012$ & $0.283\pm0.015$ & $\mathbf{0.211\pm0.015}$ & --- & --- \\
FD002 & $\mathbf{0.218\pm0.005}$ & $0.223\pm0.006$ & $0.248\pm0.008$ & --- & --- \\
FD003 & $\mathbf{0.212\pm0.010}$ & $0.287\pm0.019$ & $0.308\pm0.022$ & $0.912\pm0.009$ & $0.921\pm0.000$ \\
FD004 & $\mathbf{0.201\pm0.005}$ & $0.218\pm0.007$ & $0.239\pm0.010$ & $0.714\pm0.002$ & $0.717\pm0.003$ \\
\bottomrule
\end{tabular}
\end{table}
Two observations summarize the full-trajectory results. First, TD$(n,\lambda)$ leads MC on all four subsets in overall NAE (Table~\ref{tab:cmapss_fulltraj}), consistent with the simulation ranking. By remaining-life bucket, TD mostly leads MC and HI after the earliest 20\% of observed life. Second, HI is best on FD001 but trails TD on FD002--FD004; its FD001 gain is concentrated in the 100--80\% bucket, before trajectory evidence accumulates.
\label{par:supervision_caveat}The F1 columns compare different supervision. TD trains a Bellman-propagated mode-GVF head without explicit mode-label loss on nonterminal windows; MC trains on the clustering-derived reference labels from Section~\ref{sec:exp_cmapss} \citep{fu2025degradation}. The same reference labels score both methods, and TD's terminal/near-terminal anchors use the same pseudo-mode labels. Both subsets are ties within noise.

\begin{figure}[htbp]
  \centering
  \includegraphics[width=\textwidth]{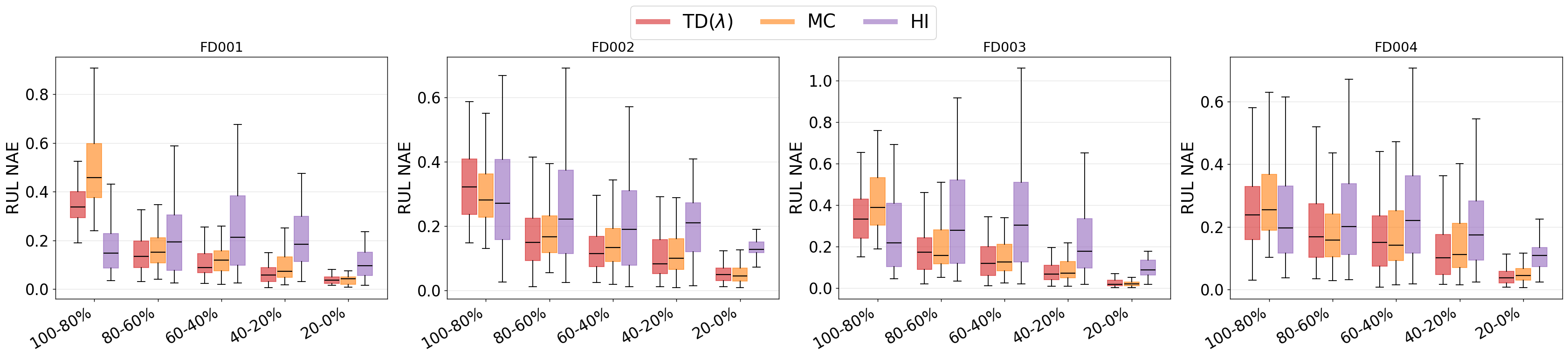}
  \caption{C-MAPSS full-data RUL NAE by remaining-life bucket (TD$(n,\lambda)$, MC, and HI).
    After the earliest stage, TD$(n,\lambda)$ typically maintains lower error than MC and HI.}
  \label{fig:cmapss_rul_boxplot}
\end{figure}

\subsection{Label-Scarce Stitch Results}
\label{sec:exp_cmapss_stitch}

Table~\ref{tab:cmapss_stitch_main} reports stitch RUL NAE and endpoint minority-class F1
(FD003/FD004) for TD$(n,\lambda)$ and the same-backbone MC control at
$f \in \{10\%,20\%,50\%\}$ under the anonymous-segment protocol of
Section~\ref{sec:exp_protocol}.
Complete results including the HI baseline, collapse diagnostics, and by-bucket
fraction plots appear in Appendix~\ref{app:cmapss_stitch_full}.

TD leads MC on every subset and fraction in Table~\ref{tab:cmapss_stitch_main}.
At $f{=}50\%$, TD stitch RUL remains comparable to full-data MC baselines on
FD003/FD004.
Endpoint minority-F1 is reported with classification loss masked on anonymous
nonterminal windows (leaving it unmasked silently inflates accuracy).
TD's mode head is stable across the segment-pool sweep, with smaller SEM than MC,
and has the higher mean F1 in every dataset--fraction cell.
Under the same stitch protocol, MC trains ordinary supervised classification only
on terminal-containing segments, whereas TD propagates mode-hitting targets from
terminal anchors by Bellman recursion on nonterminal windows.
Stitch F1 should not be compared directly to the full-data endpoint results in
Table~\ref{tab:cmapss_fulltraj}: training data and evaluation windows differ.

\begin{table}[htbp]
\centering
\caption{C-MAPSS label-scarce stitch RUL NAE ($\downarrow$) and minority-class F1 ($\uparrow$)
for TD$(n,\lambda)$ and same-backbone MC (mean~$\pm$~SEM across $3$ seeds).
\textbf{Bold} = best mean per dataset--fraction; HI and full results in Appendix~\ref{app:cmapss_stitch_full}.}
\label{tab:cmapss_stitch_main}
\setlength{\tabcolsep}{4pt}
\begin{tabular}{@{}llccc@{}}
\toprule
Dataset & Method & $f{=}10\%$ & $f{=}20\%$ & $f{=}50\%$ \\
\midrule
\multicolumn{5}{@{}l@{}}{\emph{RUL NAE $\downarrow$}} \\
FD001 & TD & $\mathbf{0.228\pm0.010}$ & $\mathbf{0.216\pm0.008}$ & $\mathbf{0.295\pm0.069}$ \\
      & MC & $0.410\pm0.034$ & $0.433\pm0.040$ & $0.310\pm0.019$ \\
FD002 & TD & $\mathbf{0.257\pm0.002}$ & $\mathbf{0.243\pm0.006}$ & $\mathbf{0.238\pm0.003}$ \\
      & MC & $0.633\pm0.016$ & $0.593\pm0.015$ & $0.597\pm0.014$ \\
FD003 & TD & $\mathbf{0.248\pm0.004}$ & $\mathbf{0.238\pm0.006}$ & $\mathbf{0.261\pm0.023}$ \\
      & MC & $0.653\pm0.002$ & $0.614\pm0.038$ & $0.551\pm0.039$ \\
FD004 & TD & $\mathbf{0.244\pm0.005}$ & $\mathbf{0.231\pm0.006}$ & $\mathbf{0.211\pm0.008}$ \\
      & MC & $0.618\pm0.012$ & $0.598\pm0.019$ & $0.612\pm0.006$ \\
\midrule
\multicolumn{5}{@{}l@{}}{\emph{Minority-class F1 $\uparrow$}} \\
FD003 & TD & $\mathbf{0.901\pm0.011}$ & $\mathbf{0.918\pm0.003}$ & $\mathbf{0.921\pm0.000}$ \\
      & MC & $0.718\pm0.067$ & $0.666\pm0.052$ & $0.773\pm0.066$ \\
FD004 & TD & $\mathbf{0.713\pm0.002}$ & $\mathbf{0.712\pm0.001}$ & $\mathbf{0.715\pm0.002}$ \\
      & MC & $0.639\pm0.030$ & $0.673\pm0.018$ & $0.674\pm0.021$ \\
\bottomrule
\end{tabular}
\end{table}

TD reproduces the full-trajectory RUL ranking and reaches full-supervision-competitive
stitch RUL without complete returns or unit identity.
The magnitude of the TD--MC gap varies across datasets and segment-pool fractions.
C-MAPSS supports, but does not universally confirm, the synthetic claim.
Additional simulation and C-MAPSS embedding visualizations are provided in
Appendix~\ref{app:umap}.

\section{Conclusion}
\label{sec:conclusion}

We formulate dynamic prognostics as vector-GVF prediction on an absorbing degradation process, so RUL and failure-mode outputs are temporally extended targets on a shared trajectory rather than independent supervised labels.
The TD($n,\lambda$) estimator learns these targets from bootstrapped Bellman transitions and therefore uses partial or anonymous trajectory prefixes that lack complete returns.
Empirically, under the label-scarce stitch protocol, TD remains competitive where same-backbone MC and HI degrade or collapse, while remaining comparable to MC under full run-to-failure supervision.

Several limitations remain. TD$(n,\lambda)$ requires selecting $n$ and $\lambda$. Within-run checkpoints use validation data; across-configuration C-MAPSS choices use minimum test-endpoint RMSE because no disjoint split exists (Section~\ref{sec:exp_cmapss_preproc}), so reported C-MAPSS results are benchmark-tuned and may be optimistic. The same rule is applied to TD, MC, and HI, but fully held-out selection at both stages would be preferable. Probability calibration of the mode head warrants further study. The failure set $\cF$ and mode-specific terminals $\cF_k$ must be defined operationally. Future work includes connecting GVF predictions to maintenance decision optimization. 

\section*{Disclosure Statement}
The authors report there are no competing interests to declare.

\section*{Declaration of Generative AI and AI-Assisted Technologies}
The authors disclose the following use of generative AI tools in preparing this work. A combination of \textbf{Claude Code} with \textbf{Claude Opus~4.8} and \textbf{Sonnet~5.0} (Anthropic), \textbf{ChatGPT 5.5, 5.6 Pro, and Codex 5.5} under a \textbf{ChatGPT Pro} subscription (OpenAI), Cursor (mostly \textbf{Composer 2.5 Fast} model) was used to (i)~assist with English-language proofreading, copy-editing, and stylistic polishing of manuscript prose, and (ii)~assist with drafting and debugging experiment and analysis code used for the simulation and C-MAPSS studies. The authors reviewed, edited, and take full responsibility for the accuracy and integrity of all content, including text, equations, code, and reported findings. 

\section*{Data Availability Statement}
{\sloppy
The NASA C-MAPSS dataset is publicly available from NASA's prognostics data repository
(\url{https://ti.arc.nasa.gov/tech/dash/groups/pcoe/prognostic-data-repository/});
the simulation study uses a synthetic event-triggered degradation process described in
Section~\ref{sec:sim_study}, with the exact generating parameters given in
Appendix~\ref{app:sim_dgp_constants}.
Code, configuration files, and seed-level outputs supporting the reported tables and
figures are provided in the accompanying reproducibility repository that will be
released once the paper is accepted.
\par}

\printbibliography

\clearpage
\appendix

\section{TD($n,\lambda$) Training Pseudocode}
\label{app:weighted_tdn}

Algorithm~\ref{alg:weighted_tdn} gives the training loop corresponding to
Section~\ref{sec:td_estimator}: store transition subsequences, form vector
TD$(n,\lambda)$ targets from
Eqs.~\eqref{eq:vector_tdn_target}--\eqref{eq:vector_tdnl_target}, and update
the shared network by a semi-gradient step on Eq.~\eqref{eq:total_loss}.

\begin{algorithm}[htbp]
\caption{TD($n,\lambda$) learning for vector GVFs.}
\label{alg:weighted_tdn}
\begin{algorithmic}[1]
\Require replay buffer $\cD$, truncation length $n$, trace parameter $\lambda$, learning rate $\alpha$, mini-batch size $B$.
\State Initialize vector-GVF network $V_\theta$ with shared encoder and task heads, and initialize replay buffer $\cD$.
\For{each observed trajectory segment}
    \State Construct windowed states $S_t$ and transition subsequences $(S_t,u_{t+1},\Gamma_{t+1},\ldots,S_{t+n})$.
    \State Store available subsequences and terminal information in $\cD$.
\EndFor
\If{$|\cD|\ge B$}
    \State Sample a mini-batch from $\cD$.
    \For{each mini-batch sample}
        \State Compute $\{y_t^{(k)}\}_{k=1}^{n}$ from Eq.~\eqref{eq:vector_tdn_target} (with $n\leftarrow k$), then compute $y_t^{(n,\lambda)}$ from Eq.~\eqref{eq:vector_tdnl_target}.
        \State Update $\theta$ by a semi-gradient step on the weighted vector-GVF loss in Eq.~\eqref{eq:total_loss}.
    \EndFor
\EndIf
\end{algorithmic}
\end{algorithm}

\section{Health-Index Baseline Implementation}
\label{app:hi_baseline}
\label{sec:exp_hi_methods}

To complement the GVF estimators with a trajectory-modeling baseline, we include a deep health-index method, referred to as \emph{HI} in the tables. It uses a one-dimensional convolutional encoder over sliding sensor windows, implemented via the same CNN1D architecture (F$\rightarrow$256$\rightarrow$96$\rightarrow$32, kernel size 7, ReLU, global average pooling, FC $32\rightarrow 64\rightarrow 128$) as the TD/MC heads, with a learned health-index output and a terminal-threshold variance penalty to encourage monotonicity. It uses a scalar degradation index $h_t = f_\theta(x_t)$ from the sliding sensor window $x_t$. Training encourages monotone degradation, stable terminal thresholds, and smooth trajectories. After health-index construction, prognostics follows the classical two-step logic: fit a degradation model in health-index space on training trajectories, then update unit-level posteriors from partial observations to infer residual life. The HI baseline infers RUL by Brownian first-passage to a learned failure threshold. Its hyperparameters are selected under the same training budget and per-platform selection protocol as the other methods---within-run epoch checkpoints on validation data, and the canonical configuration across hyperparameter settings ($\bar\gamma$, $n$, $\lambda$, window length, and early-stopping patience) selected as described for each case study (Sections~\ref{sec:exp_protocol} and~\ref{sec:exp_cmapss_preproc}). All methods share identical preprocessing and data splits.

\section{C-MAPSS Classification by Remaining-Life Bucket}
\label{app:cmapss_cls_bucket}

Figure~\ref{fig:cmapss_cls_barplot} reports full-trajectory failure-mode classification
accuracy by remaining-life bucket on FD003/FD004 (mask-corrected evaluation),
complementing the endpoint minority-class F1 in Table~\ref{tab:cmapss_fulltraj}.

\begin{figure}[htbp]
  \centering
  \includegraphics[width=\textwidth]{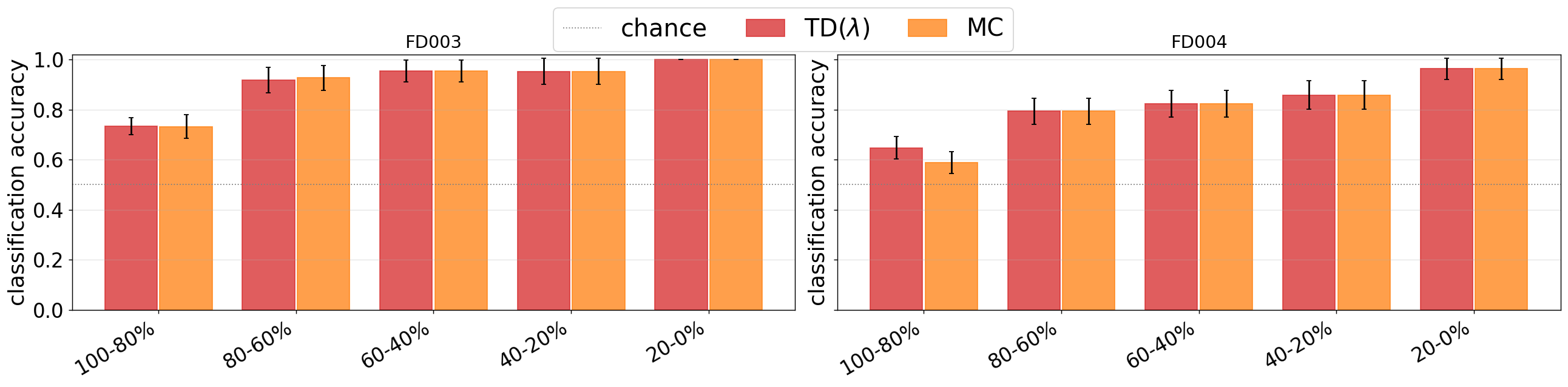}
  \caption{C-MAPSS full-data failure-mode classification accuracy by remaining-life bucket
    (FD003 and FD004; mask-corrected evaluation).
    TD uses the separately trained $\lambda{=}0$ joint model from the full-trajectory
    $\lambda$ sweep (Section~\ref{sec:exp_cmapss_fulltraj}); MC uses the same joint architecture.}
  \label{fig:cmapss_cls_barplot}
\end{figure}

\section{Complete C-MAPSS Label-Scarce Stitch Results}
\label{app:cmapss_stitch_full}

Table~\ref{tab:cmapss_stitch_full} expands Table~\ref{tab:cmapss_stitch_main} to include
the HI baseline.
MC and HI are prone to collapse under sparse anonymity (MC in $17\%$ of runs,
including every FD003 $f{=}10\%$ seed with only $14$ terminal-observing segments;
HI in $35\%$ of runs, including every FD002 $f{=}50\%$ seed).
The dagger mark on FD001 TD at $f{=}50\%$ flags elevated seed variance for that cell.
Figures~\ref{fig:cmapss_stitch_frac_boxplot} and~\ref{fig:cmapss_stitch_f1_fd003}
give by-bucket RUL and stitch F1 displays.

\begin{table}[htbp]
\centering
\caption{Complete C-MAPSS label-scarce stitch results, including the
health-index baseline, for all subsets and data fractions.
Each cell is one fleet-level summary per training seed, reported as
mean~$\pm$~SEM across three seeds ($\mathrm{SEM}=s/\sqrt{3}$).
Bold denotes the best mean within each dataset--fraction comparison among the
methods shown.
$^{\ddagger}$Elevated seed variance for that TD cell.}
\label{tab:cmapss_stitch_full}
\begin{tabular}{@{}llccc@{}}
\toprule
\multicolumn{5}{@{}l@{}}{\emph{Stitch RUL NAE $\downarrow$}} \\
Dataset & Method & $f{=}10\%$ & $f{=}20\%$ & $f{=}50\%$ \\
\midrule
FD001 & TD & $\mathbf{0.228\pm0.010}$ & $\mathbf{0.216\pm0.008}$ & $\mathbf{0.295\pm0.069}^{\ddagger}$ \\
FD001 & MC & $0.410\pm0.034$ & $0.433\pm0.040$ & $0.310\pm0.019$ \\
FD001 & HI & $1.772\pm0.781$ & $1.942\pm1.314$ & $0.711\pm0.005$ \\
FD002 & TD & $\mathbf{0.257\pm0.002}$ & $\mathbf{0.243\pm0.006}$ & $\mathbf{0.238\pm0.003}$ \\
FD002 & MC & $0.633\pm0.016$ & $0.593\pm0.015$ & $0.597\pm0.014$ \\
FD002 & HI & $1.573\pm0.983$ & $0.620\pm0.078$ & $0.703\pm0.003$ \\
FD003 & TD & $\mathbf{0.248\pm0.004}$ & $\mathbf{0.238\pm0.006}$ & $\mathbf{0.261\pm0.023}$ \\
FD003 & MC & $0.653\pm0.002$ & $0.614\pm0.038$ & $0.551\pm0.039$ \\
FD003 & HI & $0.591\pm0.107$ & $1.409\pm0.609$ & $1.316\pm0.690$ \\
FD004 & TD & $\mathbf{0.244\pm0.005}$ & $\mathbf{0.231\pm0.006}$ & $\mathbf{0.211\pm0.008}$ \\
FD004 & MC & $0.618\pm0.012$ & $0.598\pm0.019$ & $0.612\pm0.006$ \\
FD004 & HI & $0.938\pm0.429$ & $0.849\pm0.176$ & $0.681\pm0.003$ \\
\midrule
\multicolumn{5}{@{}l@{}}{\emph{Stitch minority-class F1 $\uparrow$}} \\
Dataset & Method & $f{=}10\%$ & $f{=}20\%$ & $f{=}50\%$ \\
\midrule
FD003 & TD & $\mathbf{0.901\pm0.011}$ & $\mathbf{0.918\pm0.003}$ & $\mathbf{0.921\pm0.000}$ \\
FD003 & MC & $0.718\pm0.067$ & $0.666\pm0.052$ & $0.773\pm0.066$ \\
FD004 & TD & $\mathbf{0.713\pm0.002}$ & $\mathbf{0.712\pm0.001}$ & $\mathbf{0.715\pm0.002}$ \\
FD004 & MC & $0.639\pm0.030$ & $0.673\pm0.018$ & $0.674\pm0.021$ \\
\bottomrule
\end{tabular}
\end{table}

\begin{figure}[htbp]
  \centering
  \includegraphics[width=\linewidth]{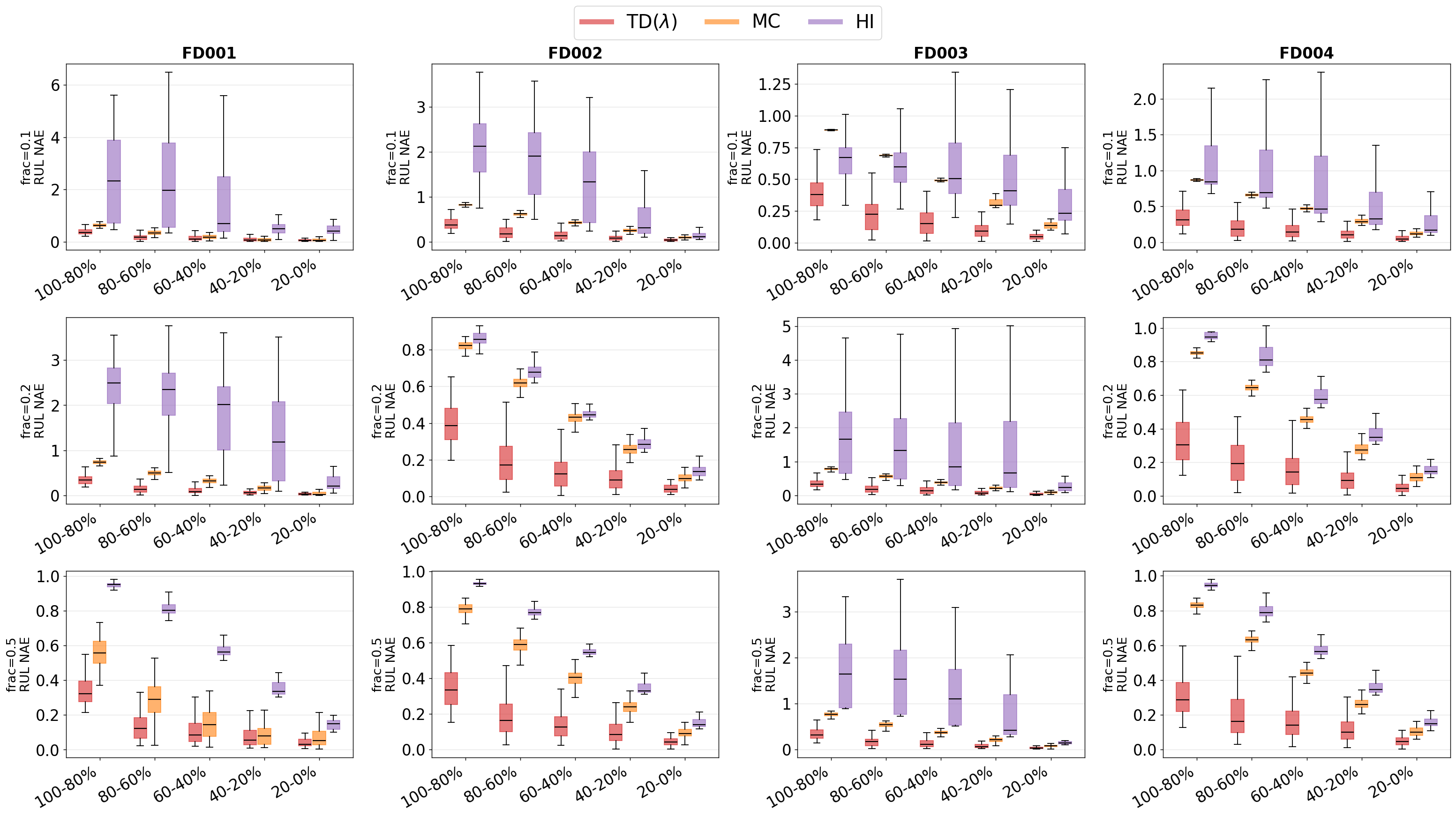}
  \caption{C-MAPSS label-scarce stitch: RUL NAE by remaining-life bucket across fractions $f$ and subsets FD001--FD004 for TD$(n,\lambda)$, MC, and HI.}
  \label{fig:cmapss_stitch_frac_boxplot}
\end{figure}

\begin{figure}[htbp]
  \centering
  \includegraphics[width=\linewidth]{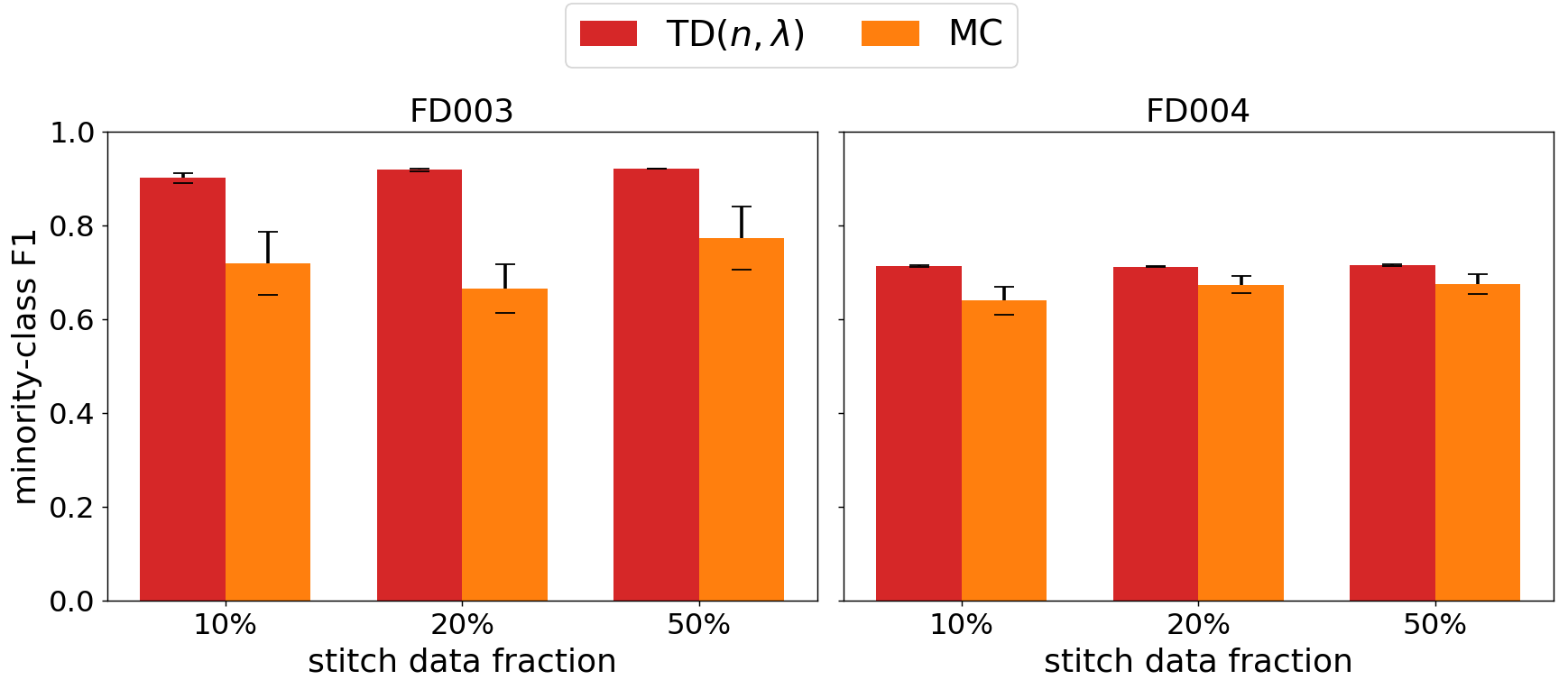}
  \caption{C-MAPSS stitch minority-class F1 for FD003/FD004,
    mean~$\pm$~SEM across training seeds
    (TD: $n{=}3$ throughout; MC: $n{=}3$ on FD003 and $n{=}6,8,8$ on FD004).
    TD is flat across fractions; MC has larger spread, especially on FD003.}
  \label{fig:cmapss_stitch_f1_fd003}
\end{figure}

\section{Embedding Visualizations (UMAP)}
\label{app:umap}
\label{sec:exp_cmapss_umap}

UMAP projections below are qualitative and parameter-dependent; they are not used as primary evidence.
Reported purity scores use 20-nearest-neighbor labels in the displayed UMAP coordinates and should be read as visual diagnostics only.

\subsection{Simulation Embeddings}

Figure~\ref{fig:umap_full_4models} shows two-dimensional UMAP projections of the
penultimate-layer CNN embeddings for all four methods on the simulation test set.
TD$(n,\lambda)$ maintains high failure-mode purity (0.80) in both full-data and stitch
settings; MC drops to near chance under stitch (0.55). In these projections, TD preserves
failure-mode structure more effectively under stitch anonymity than MC.

\begin{figure}[htbp]
  \centering
  \includegraphics[width=\linewidth]{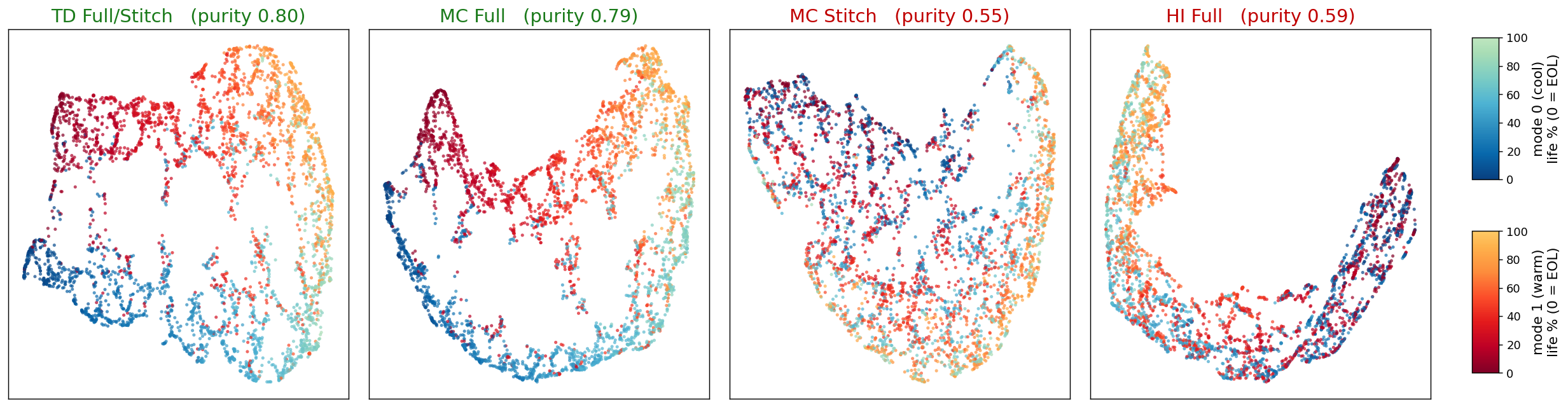}
  \caption{Simulation UMAP embeddings colored by remaining life. Panel purities:
    TD$(n,\lambda)$ full/stitch pooled 0.80, MC full 0.79, MC stitch 0.55, HI full 0.59.
    UMAP is visual only; purity uses 20-nearest-neighbor labels in UMAP space.}
  \label{fig:umap_full_4models}
\end{figure}

\subsection{C-MAPSS Embeddings}

Figure~\ref{fig:cmapss_umap} shows bivariate UMAP projections of penultimate CNN embeddings for TD$(n,\lambda)$ and MC on FD003/FD004 under $f{=}100\%$ and $f{=}50\%$ stitch fractions. Points encode failure mode by hue and normalized remaining life by shade. TD$(n,\lambda)$ forms a clearer life-stage gradient in these displays, consistent with the adjacent-window bootstrap mechanism evaluated quantitatively in Section~\ref{sec:exp_cmapss}.

\begin{figure}[htbp]
  \centering
  \includegraphics[width=0.49\linewidth]{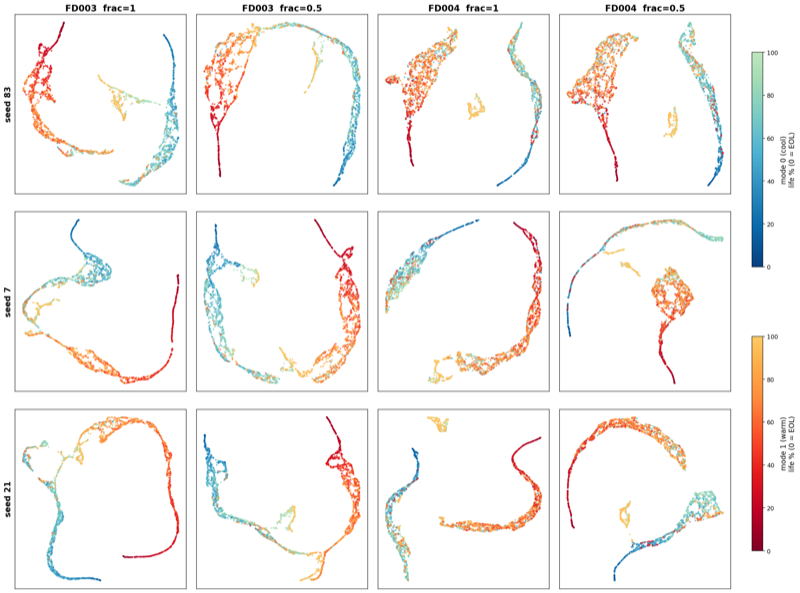}%
  \hfill%
  \includegraphics[width=0.49\linewidth]{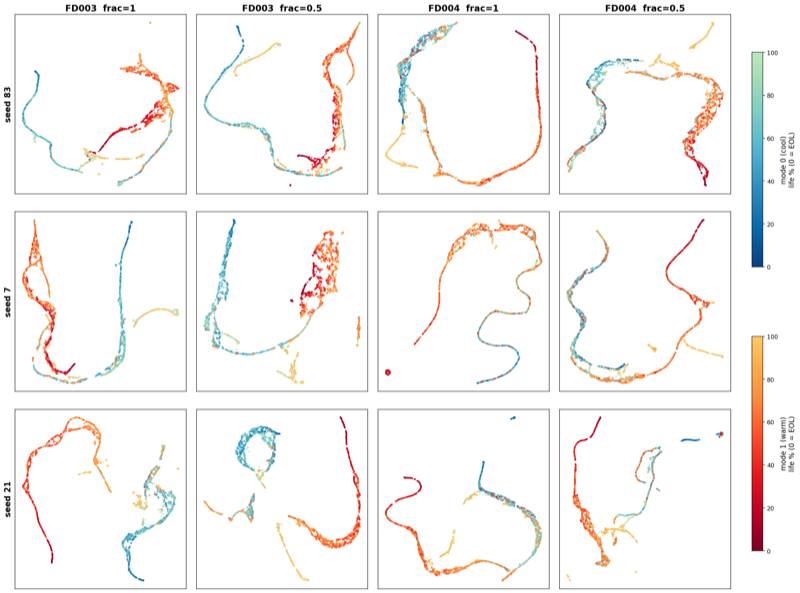}
  \caption{C-MAPSS UMAP embeddings on FD003/FD004, colored by failure mode and remaining life.
    Left: TD$(n,\lambda)$; right: MC; seeds $83$, $7$, $21$.
    Columns: FD003 $f{=}100\%$, FD003 $f{=}50\%$, FD004 $f{=}100\%$, FD004 $f{=}50\%$.}
  \label{fig:cmapss_umap}
\end{figure}

\section{Simulation Data-Generating Process: Exact Parameters}
\label{app:sim_dgp_constants}

Table~\ref{tab:sim_dgp_constants} lists the exact parameter values, priors, and sensor
construction formulas for the event-triggered simulation (ET-Long) of
Section~\ref{sec:sim_study}, supplementing the description given there.

\begin{table}[htbp]
\centering
\caption{ET-Long simulation data-generating process: exact parameter values.
Both failure-mode configurations use linear sensor fusion.
The failure modes differ in both the late-phase degradation kernel and the
post-trigger discriminative pattern.}
\label{tab:sim_dgp_constants}
\begin{tabular}{@{}ll@{}}
\toprule
Parameter & Value \\
\midrule
Sensor noise sd $\sigma_x$ & $0.05$ (all channels) \\
HI-process noise sd $\sigma_{hi}$ & $0.05$ (Phases A and B) \\
Channel coeffs.\ $a_{11},a_{21},a_{41}$ & $\mathrm{Unif}(0,2)$, $\mathrm{Unif}(0,1)$, $\mathrm{Unif}(0,1)$ \\
Channel coeffs.\ $a_{12},a_{22},a_{42}$ & $\mathrm{Unif}(0,0.5)$ each \\
Activation threshold $\theta$ & $\mathrm{Unif}(0.4,0.6)$, clipped to $[0.05,0.95]$ \\
Failure threshold $B_{\mathrm{threshold}}$ & $5.0$ \\
Phase-B amplitude scale $\rho_B$ & $1.8$ \\
Time step $dt$ & $0.005$ \\
Phase-A / Phase-B horizons & $2.0$ / $10.0$ \\
\midrule
\multicolumn{2}{@{}p{0.95\linewidth}@{}}{%
Early-phase $g_A(t)$ (both modes):
$\gamma_0 + \gamma_1(e^t-1) + \gamma_2 t^3$,
$\gamma \sim \mathcal{N}([0.0,0.5,1.5],\,\Sigma_A)$.} \\
\multicolumn{2}{@{}p{0.95\linewidth}@{}}{%
Late-phase Mode~1 (quadratic):
$\gamma_0 + \gamma_1 s + \gamma_2 s^2$,
$\gamma \sim t_{\nu=5}([0.2,1.0,2.0],\,\Sigma_{B,\mathrm{quad}})$.} \\
\multicolumn{2}{@{}p{0.95\linewidth}@{}}{%
Late-phase Mode~2 (exp-cubic): same functional form and prior as the early phase.} \\
\bottomrule
\end{tabular}
\end{table}

with
$\Sigma_A = \left(\begin{smallmatrix}0.01&0.001&0.001\\0.001&0.1&0.005\\0.001&0.005&0.2\end{smallmatrix}\right)$
and
$\Sigma_{B,\mathrm{quad}} = \left(\begin{smallmatrix}0.01&0.001&0.001\\0.001&0.1&0.05\\0.001&0.05&0.2\end{smallmatrix}\right)$
(late-phase Mode~2 reuses $\Sigma_A$). Sensor construction (linear fusion):
$x_1 = a_{11}t^2 - a_{12}\sin(25t) + \mathcal{N}(0,\sigma_x^2)$,
$x_2 = a_{21}t + a_{22}\sin(50t) + \mathcal{N}(0,\sigma_x^2)$,
$x_4 = a_{41}t + a_{42} + \mathcal{N}(0,\sigma_x^2)$; the vector enrichment is
$\boldsymbol{p}(s)=p_0(u)\,(1,0,0,0)^\top$ from Eq.~\eqref{eq:dgp}, so only $x_1$
receives the mode-discriminative post-trigger shape on top of this base additive DGP.
Mode~1 uses the monotone pattern $\alpha u^{1.5}$; Mode~2 uses the oscillatory
$\sqrt{u}$-enveloped pattern
$\mathrm{amp}\cdot\sin(2\pi\,n_{\mathrm{cyc}}\cdot\mathrm{norm}+\phi_p)\sqrt{\mathrm{norm}}
+\mathrm{amp}\cdot0.15\cdot\mathrm{norm}$, with a random phase $\phi_p\sim\mathrm{Unif}(0,2\pi)$
and a small mean shift growing with the normalized post-trigger time $\mathrm{norm}$.
Independently of this post-trigger enrichment, both modes also receive a shared
pre-trigger enrichment on $x_1,x_2$: a slow sinusoid, a fast sinusoid, and a linear
drift term ($A_{\mathrm{slow}}\sin(2\pi t/T_{\mathrm{slow}}+\phi_{\mathrm{slow}})
+A_{\mathrm{fast}}\sin(2\pi t/T_{\mathrm{fast}}+\phi_{\mathrm{fast}})+\mathrm{drift}\cdot t$),
drawn from the same distribution for both modes so it does not itself carry
mode-discriminative information before the trigger.

\section{Proofs for the Statistical Properties of the Estimator}
\label{app:proof_algorithmic_properties}

Results that are
standard in the TD literature are proved by verifying that our setting satisfies the
hypotheses of the corresponding cited theorems; results specific to this paper are
proved in full. Table~\ref{tab:notation} summarizes the notation used throughout.

\begin{table}[htbp]
\centering
\caption{Notation summary.}
\label{tab:notation}
\begin{tabular}{ll}
\toprule
Symbol & Meaning \\
\midrule
$\cS$ & Degradation state space \\
$S_t$ & State at time $t$ \\
$\cF$ & Absorbing failure set \\
$\cF_k$ & Terminal set for failure mode $k$ \\
$T$ & Absorption/failure time \\
$K$ & Number of failure modes \\
$c_{t+1}^{(i)}$ & Cumulant for GVF component $i$ \\
$\gamma_{t+1}^{(i)}$ & Continuation for GVF component $i$ \\
$V_\gamma^{\mathrm{time}}$ & Discounted survival-time (RUL) GVF \\
$H_\gamma^{(k)}$ & Hitting-probability GVF for failure mode $k$ \\
$u_{t+1}$ & Vector cumulant \\
$\Gamma_{t+1}$ & Diagonal continuation matrix \\
$\Pi_\Psi$ & $L_2(\Psi)$ projection onto feature span \\
$A_i,b_i$ & Linear projected Bellman system for component $i$ \\
$y_t^{(n)}$ & Vector TD($n$) target \\
$y_t^{(n,\lambda)}$ & Truncated vector TD($n,\lambda$) target \\
$y_t^{\mathrm{MC}}$ & Vector Monte Carlo target \\
\bottomrule
\end{tabular}
\end{table}

\subsection{Standard supporting facts}
\label{app:standard_supporting}

The following assumptions and results are standard in TD analysis
\citep{tsitsiklis1997analysis, sutton2018reinforcement}; but adapted for the vector-valued GVF formulation of Section~\ref{sec:vector_gvf}.

\begin{assumption}[A1. Ergodicity on the learning process]
\label{ass:ergodicity}
The state process used for learning is geometrically ergodic with stationary distribution $\Psi$ (or is converted to an ergodic regenerative process by episodic reset at absorption for theoretical analysis).
\end{assumption}

\begin{assumption}[A2. Bounded cumulants]
\label{ass:bounded_cumulant}
Each GVF cumulant component is bounded almost surely: $|u_{t+1}^{(i)}|\le \bar c_i<\infty$.
\end{assumption}

\begin{assumption}[A3. Bounded features]
\label{ass:bounded_features}
For linear analysis, the feature map satisfies $\phi:\cS\to\R^d$ with $\|\phi(s)\|_2\le 1$.
\end{assumption}

\begin{assumption}[A4. Well-posed projected equations]
\label{ass:well_posed_projected}
For each GVF component $i$, the projected fixed-point equation for the $\lambda$-return operator has a unique, stable solution under $\Psi$: the symmetric part $\tfrac12(A_{i,n,\lambda}+A_{i,n,\lambda}^\top)$ of the corresponding (generally nonsymmetric) TD normal matrix $A_{i,n,\lambda}$ is positive definite, a sufficient condition for $-A_{i,n,\lambda}$ to be Hurwitz. (Mere nonsingularity of $A_{i,n,\lambda}$ is insufficient because the matrix is nonsymmetric.)
\end{assumption}

\begin{assumption}[A5. Step-size conditions]
\label{ass:step_size}
The TD step sizes satisfy $\sum_t \alpha_t=\infty$, $\sum_t \alpha_t^2<\infty$, with standard Robbins--Monro conditions.
\end{assumption}

\begin{assumption}[A6. Termination or contraction]
\label{ass:termination_or_contraction}
For each component $i$, either $\gamma_{t+1}^{(i)}\le \bar\gamma_i<1$ almost surely, or the absorbing episodic process terminates almost surely with \emph{uniformly} bounded expected absorption time, $\sup_{s\in\cC}\E_s[\tau]<\infty$ (for a finite transient set, equivalently $\rho(P_{\cC\cC})<1$), which gives finite expected return.
\end{assumption}

Under Assumptions~\ref{ass:ergodicity}--\ref{ass:step_size},
Sections~\ref{app:proof_fixed_point}--\ref{app:td_mc_realizability} prove
Theorem~\ref{thm:bellman_td_convergence}. Claim~(i) follows from the
contraction property of the component Bellman operators and a geometric
bound on the discounted return. Claim~(ii) follows from the standard
on-policy stochastic-approximation argument for the projected
TD$(n,\lambda)$ equation. Claim~(iii) follows by comparing the projected
Bellman solution with the least-squares projection of the complete
return.

The discounted contraction used in claim~(i) does not cover the
undiscounted case $\gamma\equiv 1$. That case is treated separately in
Proposition~\ref{prop:undiscounted_absorbing_fixed_point} under a
uniform finite-absorption-time condition.

\subsection{Proof of Proposition~\ref{prop:soft_horizon_bound_theory}}
\label{app:proof_soft_horizon_bound}

Fix $s\in\cC$ and $\gamma\in[0,1)$. By definition,
$V_\gamma^{\mathrm{time}}(s)=\E_s\!\left[\sum_{\ell=0}^{\tau_t-1}\gamma^\ell\right]$,
where $\tau_t\ge 1$ on $\cC$ and $\tau_t$ may be infinite on trajectories that never
reach $\cF$. Since every summand $\gamma^\ell$ is nonnegative, the random partial sum
$\sum_{\ell=0}^{\tau_t-1}\gamma^\ell$ is nonnegative pathwise, so
$V_\gamma^{\mathrm{time}}(s)\ge 0$. For the upper bound, on every sample path (whether
$\tau_t$ is finite or infinite),
\[
\sum_{\ell=0}^{\tau_t-1}\gamma^\ell
\le
\sum_{\ell=0}^{\infty}\gamma^\ell
=
\frac{1}{1-\gamma},
\]
because the geometric series with ratio $\gamma\in[0,1)$ converges and its partial sums
are monotone increasing in the number of terms. Taking expectations under $\E_s$ and
using monotonicity of expectation yields
$V_\gamma^{\mathrm{time}}(s)\le (1-\gamma)^{-1}$. In particular, the expectation is
well defined and finite without any moment condition on $\tau_t$.
\qed

\subsection{Proof of Proposition~\ref{prop:terminal_event_mode_tilting}}
\label{app:proof_terminal_event_mode_tilting}

Fix $s\in\cC$ and $\gamma\in[0,1)$. By Definition~\ref{def:mode_terminal_event_gvf},
\[
H_\gamma^{(k)}(s)
=
\E_s\!\left[\gamma^{\tau_t-1}\one\{S_{t+\tau_t}\in\cF_k\}\right],
\]
where the integrand is set to zero on the event $\{\tau_t=\infty\}$ (no absorption), so
the expectation is well defined and bounded by $1$.

\emph{Factorization.} On the absorption event, the terminal mode $Z$ is the index $k$
with $S_{t+\tau_t}\in\cF_k$, and the mode sets $\cF_1,\ldots,\cF_K$ are disjoint, so
$\one\{S_{t+\tau_t}\in\cF_k\}=\one\{Z=k\}$. Decomposing the expectation over the event
$\{Z=k\}$ and its complement (on which the integrand vanishes), and applying the
definition of conditional expectation,
\[
H_\gamma^{(k)}(s)
=
\E_s\!\left[\gamma^{\tau_t-1}\one\{Z=k\}\right]
=
\Prob_s(Z=k)\,\E_s\!\left[\gamma^{\tau_t-1}\mid Z=k\right]
=
p_k(s)\,a_k(s;\gamma),
\]
which is Eq.~\eqref{eq:mode_tilting_factorization}. (If $p_k(s)=0$, both sides are
zero and the factorization holds trivially with any convention for $a_k$.)

\emph{Limit as $\gamma\uparrow1$.} Conditional on $Z=k$, absorption occurs and
$\tau_t<\infty$ almost surely, so $\gamma^{\tau_t-1}\to 1$ pointwise as
$\gamma\uparrow1$. Since $0\le\gamma^{\tau_t-1}\le 1$, dominated convergence gives
$a_k(s;\gamma)=\E_s[\gamma^{\tau_t-1}\mid Z=k]\to 1$, and hence
$H_\gamma^{(k)}(s)=p_k(s)a_k(s;\gamma)\to p_k(s)$ for every $k$.

\emph{Normalized score.} If $\sum_{j=1}^K p_j(s)>0$, then
$\sum_{j=1}^K H_\gamma^{(j)}(s)\to\sum_{j=1}^K p_j(s)$ by the previous step, and the
denominator is strictly positive for $\gamma$ close enough to $1$. Therefore
\[
\tilde p_{\gamma,k}(s)
=
\frac{H_\gamma^{(k)}(s)}{\sum_{j=1}^K H_\gamma^{(j)}(s)}
\;\longrightarrow\;
\frac{p_k(s)}{\sum_{j=1}^K p_j(s)},
\]
which equals $p_k(s)$ whenever absorption is almost sure from $s$
(i.e., $\sum_{j} p_j(s)=1$).
\qed

\subsection{Proof of Proposition~\ref{prop:mc_td_target_variance}}
\label{app:proof_mc_td_target_variance}

Equation~\eqref{eq:mc_full_return} is the standard
prefix--residual partition of a discounted return: the first $n$ terms form $P_{t,n}$ and
the remaining discounted sum starting at time $t+n$ is $R_{t+n}$, so
$G_t^{\mathrm{MC}}=P_{t,n}+\gamma^{n}R_{t+n}$ holds pathwise on $\{\tau_t>n\}$ --- the
absorption time $\tau_t$ is itself random given only $S_t=s$, so all statements below are
conditional on the event $\{S_t=s,\,\tau_t>n\}$, on which $S_{t+n}$ is a well-defined
non-absorbed state and $R_{t+n}$ the (possibly further-absorbing) residual return from it.
Let $\mathcal H_{t,n}$ denote the sigma-field generated by the observed prefix through
$S_{t+n}$, including $S_t,S_{t+1},\ldots,S_{t+n}$ and the cumulants in $P_{t,n}$, on this
event. The prefix $P_{t,n}$ is $\mathcal H_{t,n}$-measurable.
Under the Markov predictive-state assumption,
\[
\E[R_{t+n}\mid \mathcal H_{t,n}]
=
\E[R_{t+n}\mid S_{t+n}]
=
V^\star(S_{t+n}).
\]
Therefore
\[
\E[G_t^{\mathrm{MC}}\mid \mathcal H_{t,n}]
=
P_{t,n}
+
\gamma^{n}
\E[R_{t+n}\mid \mathcal H_{t,n}]
=
P_{t,n}
+
\gamma^{n} V^\star(S_{t+n})
=
G_t^{\mathrm{TD},\star}.
\]
Applying the law of total variance conditional on $\{S_t=s,\,\tau_t>n\}$ gives
\begin{align*}
\Var(G_t^{\mathrm{MC}}\mid S_t=s,\,\tau_t>n)
&=
\Var\!\bigl(
\E[G_t^{\mathrm{MC}}\mid \mathcal H_{t,n}]
\mid S_t=s,\,\tau_t>n
\bigr) \\
&\quad+
\E\!\bigl[
\Var(G_t^{\mathrm{MC}}\mid \mathcal H_{t,n})
\mid S_t=s,\,\tau_t>n
\bigr].
\end{align*}
The first term is $\Var(G_t^{\mathrm{TD},\star}\mid S_t=s,\,\tau_t>n)$.
For the second term, $P_{t,n}$ is $\mathcal H_{t,n}$-measurable, so
\[
\Var(G_t^{\mathrm{MC}}\mid \mathcal H_{t,n})
=
\Var(P_{t,n}+\gamma^{n} R_{t+n}\mid \mathcal H_{t,n})
=
\gamma^{2n}\Var(R_{t+n}\mid \mathcal H_{t,n}).
\]
Using the Markov assumption again,
\[
\Var(R_{t+n}\mid \mathcal H_{t,n})
=
\Var(R_{t+n}\mid S_{t+n}).
\]
Substituting these identities yields Eq.~\eqref{eq:mc_td_variance_decomp}.
Since the final conditional-variance term is nonnegative, the oracle TD target has no
larger conditional variance than the MC target.
Equivalently, $G_t^{\mathrm{TD},\star}=\E[G_t^{\mathrm{MC}}\mid \mathcal H_{t,n}]$, so the
oracle TD target is a Rao--Blackwellized version of MC: it replaces the realized residual
by its conditional expectation given the observed prefix through $S_{t+n}$.
With a learned bootstrap $\widehat V=V^\star+e$, TD trades residual-return variance for
bootstrap-error variance and is advantageous when the latter is smaller.
\qed

\subsection{Proof of Theorem~\ref{thm:bellman_td_convergence}, claim~(i):
Population fixed point and value bound}
\label{app:proof_fixed_point}

Fix a GVF component $i$ with $|u_{t+1}^{(i)}|\le\bar c_i<\infty$ and
$0\le\gamma_{t+1}^{(i)}\le\bar\gamma_i<1$ almost surely. The component Bellman operator is
\begin{equation}
(\mathcal{T}_i v)(s)=\E\!\bigl[u_{t+1}^{(i)}+\gamma_{t+1}^{(i)}v(S_{t+1})\mid S_t=s\bigr].
\end{equation}
Then
\begin{equation}
\|\mathcal{T}_i v-\mathcal{T}_i w\|_\infty\le \bar\gamma_i\|v-w\|_\infty,
\end{equation}
so $\mathcal{T}_i$ is a $\bar\gamma_i$-contraction and has a unique fixed point $V_i^\star$
by Banach's theorem.

By induction on $k$, the $k$-fold composition satisfies, for any bounded $v,w$,
\[
\|\mathcal{T}_i^k v-\mathcal{T}_i^k w\|_\infty
\le
\bar\gamma_i^k\|v-w\|_\infty,
\]
so $\mathcal{T}_i^k$ is a contraction with modulus at most $\bar\gamma_i^k$. The truncated
TD$(n,\lambda)$ population operator
$\mathcal{T}_{i,n,\lambda}=(1-\lambda)\sum_{k=1}^{n-1}\lambda^{k-1}\mathcal{T}_i^k
+\lambda^{n-1}\mathcal{T}_i^n$ is a convex combination of these compositions (the weights
$(1-\lambda)\lambda^{k-1}$ for $k=1,\dots,n-1$ and $\lambda^{n-1}$ sum to one). By the
triangle inequality,
\[
\|\mathcal{T}_{i,n,\lambda}v-\mathcal{T}_{i,n,\lambda}w\|_\infty
\le
\rho_{i,n,\lambda}\|v-w\|_\infty,
\]
with
\[
\rho_{i,n,\lambda}
=
(1-\lambda)\sum_{k=1}^{n-1}\lambda^{k-1}\bar\gamma_i^k
+\lambda^{n-1}\bar\gamma_i^n
<1,
\]
because it is a convex combination of the values $\bar\gamma_i^k\le\bar\gamma_i<1$.
Hence $\mathcal{T}_{i,n,\lambda}$ is a contraction and has a unique fixed point. If
$\mathcal{T}_i V_i^\star=V_i^\star$, then $\mathcal{T}_i^k V_i^\star=V_i^\star$ for every
$k\ge1$, so $\mathcal{T}_{i,n,\lambda}V_i^\star=V_i^\star$; by uniqueness the two
operators share the same fixed point $V_i^\star$.

For the value bound, the return from $S_t=s$ is
\[
G_t^{(i)}=\sum_{\ell=0}^{\infty}\Bigl(\prod_{j=1}^{\ell}\gamma_{t+j}^{(i)}\Bigr)u_{t+\ell+1}^{(i)},
\]
with the empty product equal to $1$. Pathwise,
$\prod_{j=1}^{\ell}\gamma_{t+j}^{(i)}\le\bar\gamma_i^{\ell}$, so the series is dominated
in absolute value by $\sum_{\ell\ge0}\bar\gamma_i^{\ell}\bar c_i=\bar c_i/(1-\bar\gamma_i)$;
in particular $G_t^{(i)}$ is well defined and absolutely convergent almost surely, and
$V_i^\star(s)=\E[G_t^{(i)}\mid S_t=s]$ exists. By the triangle inequality and monotone
convergence applied to the dominating series,
\begin{equation}
|V_i^\star(s)|
\le
\E\!\left[\sum_{\ell=0}^{\infty}\Bigl(\prod_{j=1}^{\ell}\gamma_{t+j}^{(i)}\Bigr)
|u_{t+\ell+1}^{(i)}|\;\middle|\;S_t=s\right]
\le
\sum_{\ell=0}^{\infty}\bar\gamma_i^\ell \bar c_i
=\frac{\bar c_i}{1-\bar\gamma_i}.
\end{equation}
The right-hand side does not depend on $s$, so
$\sup_{s\in\cS}|V_i^\star(s)|\le \bar c_i/(1-\bar\gamma_i)$.

Because $\Gamma_{t+1}$ is diagonal, the stacked vector operator
$(\mathcal{T}V)(s)=\E[u_{t+1}+\Gamma_{t+1}V(S_{t+1})\mid S_t=s]$ decouples into the
component operators $\mathcal{T}_i$, and the same conclusions hold for the vector GVF.
\qed

\begin{proposition}[Undiscounted absorbing fixed point]
\label{prop:undiscounted_absorbing_fixed_point}
Under Assumptions~\ref{ass:bounded_cumulant} and~\ref{ass:termination_or_contraction}
with $\gamma_{t+1}^{(i)}\equiv 1$ on the transient set, if
$\sup_{s\in\cC}\E_s[\tau]<\infty$ (for a finite transient set, equivalently
$\rho(P_{\cC\cC})<1$), then the component Bellman operator $\mathcal{T}_i$ has a unique
bounded fixed point $V_i^\star$ on $\cC$ with terminal boundary $V_i^\star\equiv 0$ on
$\cF$.
\end{proposition}

\begin{proof}
Write the return on the transient set as
\begin{equation}
G_t^{(i)}=\sum_{\ell=0}^{T-t-1} u_{t+\ell+1}^{(i)},
\end{equation}
which satisfies $|G_t^{(i)}|\le \bar c_i\,(T-t)$ pathwise. Per-state finiteness of
$\E_s[\tau]$ is \emph{not} sufficient for the sup-norm uniqueness argument below; we
require the \emph{uniform} condition $\sup_{s\in\cC}\E_s[\tau]<\infty$. For a finite
transient set this is equivalent to $\rho(P_{\cC\cC})<1$, where $P_{\cC\cC}$ is the
sub-stochastic transition matrix restricted to $\cC$, in which case the fundamental
matrix $(I-P_{\cC\cC})^{-1}$ exists and $V_i^\star=(I-P_{\cC\cC})^{-1}u^{(i)}$ on $\cC$.
Under this uniform condition the return is integrable with bounded expectation.
Define $V_i^\star(s)=\E[G_t^{(i)}\mid S_t=s]$ on the transient set, with the terminal
boundary condition $V_i^\star(s)=0$ for $s\in\cF$. Splitting off the first transition
and using the Markov property,
\[
V_i^\star(s)
=
\E\bigl[u_{t+1}^{(i)}\mid S_t=s\bigr]
+
\E\bigl[V_i^\star(S_{t+1})\mid S_t=s\bigr]
=
(\mathcal{T}_i V_i^\star)(s),
\]
where the interchange of the expectation and the tail sum is justified by dominated
convergence with dominating variable $\bar c_i\,(T-t)$, which is integrable. Thus
$V_i^\star$ is a fixed point. For uniqueness, let $v$ be any bounded fixed point with
$v\equiv 0$ on $\cF$, and set $h=v-V_i^\star$. Then $h(s)=\E[h(S_{t+1})\mid S_t=s]$, and
iterating $m$ times gives $h(s)=\E[h(S_{t+m})\one\{T>t+m\}\mid S_t=s]$ because $h$
vanishes on the absorbing set. Since $T<\infty$ almost surely and $h$ is bounded,
$|h(s)|\le\|h\|_\infty\,\Prob_s(T>t+m)\to 0$ as $m\to\infty$, so $h\equiv 0$ and the
fixed point is unique on the transient set. (This is the standard absorbing-chain
potential argument; see \citet[Ch.~3--4]{sutton2018reinforcement} for the episodic
formulation.)
\end{proof}

\subsection{Proof of Theorem~\ref{thm:bellman_td_convergence}, claim~(ii):
Linear TD$(n,\lambda)$ convergence}
\label{app:proof_projected_vector_td}

This is a special case of the on-policy linear TD($\lambda$) convergence
theorem of \citet{tsitsiklis1997analysis}, applied separately to each GVF component;
we only verify that its hypotheses hold in our setting. Because $\Gamma_{t+1}$ is
diagonal, the multi-output linear approximation $V_W(s)=W^\top\phi(s)$ with
$W=[w_1,\dots,w_m]\in\R^{d\times m}$ decouples into $m$ independent scalar linear-TD
recursions, one per column $w_i$, each driven by the scalar cumulant $u^{(i)}$ and
continuation $\gamma^{(i)}$. For each component: (i) the sampling process is ergodic
with stationary distribution $\Psi$ (Assumption~\ref{ass:ergodicity}); (ii) the
cumulants and features are bounded (Assumptions~\ref{ass:bounded_cumulant}
and~\ref{ass:bounded_features}), giving the required moment conditions; (iii) the step
sizes satisfy the Robbins--Monro conditions (Assumption~\ref{ass:step_size}); and
(iv) the projected fixed-point system $A_{i,n,\lambda}w_i=b_{i,n,\lambda}$ for the
truncated $\lambda$-return operator is stable in the sense made precise below. Two points
require care. First, the relevant normal matrix is \emph{not} the one-step TD(0) matrix
$\E_\Psi[\phi_t(\phi_t-\gamma_t^{(i)}\phi_{t+1})^\top]$, but the matrix of the
$\lambda$-return operator $\mathcal{T}_{i,n,\lambda}$ of Eq.~\eqref{eq:tdnl_operator}
that we actually train against,
\[
A_{i,n,\lambda}=\E_\Psi\!\bigl[\phi_t\,(\phi_t-\psi_t^{(i,n,\lambda)})^\top\bigr],
\qquad
b_{i,n,\lambda}=\E_\Psi\!\bigl[\phi_t\,g_t^{(i,n,\lambda)}\bigr],
\]
where $\psi_t^{(i,n,\lambda)}$ is the $\lambda$-weighted, $n$-truncated discounted
successor feature obtained by applying the same convex weights
$\{(1-\lambda)\lambda^{k-1}\}_{k=1}^{n-1}\cup\{\lambda^{n-1}\}$ that define
$\mathcal{T}_{i,n,\lambda}$ to the $k$-step discounted features
$\E_\Psi[\prod_{j=1}^{k}\gamma_{t+j}^{(i)}\,\phi_{t+k}\mid\mathcal H_t]$, and
$g_t^{(i,n,\lambda)}$ is the matching $\lambda$-weighted partial cumulant. Because each
$n$-step operator $\mathcal{T}_i^k$ is a $\bar\gamma_i$-contraction and
$\mathcal{T}_{i,n,\lambda}$ is a convex combination of them (claim~(i)), the standard
on-policy projected-TD argument for the $\lambda$-return operator
\citep{tsitsiklis1997analysis} applies to $A_{i,n,\lambda}$. Second, the required
stability condition is not mere nonsingularity: TD normal matrices are generally
nonsymmetric, and nonsingularity of a nonsymmetric matrix does not imply the stability
needed for the stochastic-approximation ODE. We therefore require that the symmetric part
$\tfrac12(A_{i,n,\lambda}+A_{i,n,\lambda}^\top)$ of $A_{i,n,\lambda}$ be positive definite,
a sufficient condition for $-A_{i,n,\lambda}$ to be Hurwitz (all eigenvalues have strictly
negative real part; Assumption~\ref{ass:well_posed_projected}); this is precisely what
on-policy sampling under the $\bar\gamma_i$-contraction $\mathcal{T}_{i,n,\lambda}$
guarantees. Under
Assumption~\ref{ass:termination_or_contraction} the discounting/termination condition of
\citet{tsitsiklis1997analysis} holds as well. The cited theorem then yields
$w_{i,t}\to w_i^\star=A_{i,n,\lambda}^{-1}b_{i,n,\lambda}$ almost surely for each $i$,
i.e., convergence of the stacked iterate $W_t$ to the unique projected TD($n,\lambda$)
fixed point, which is claim~(ii) of Theorem~\ref{thm:bellman_td_convergence}. \qed

\subsection{Proof of Theorem~\ref{thm:bellman_td_convergence}, claim~(iii):
Realizability and the TD--MC relation}
\label{app:td_mc_realizability}

Fix a scalar component $i$ and write $v_w(s)=\phi(s)^\top w$, assuming without loss of
generality that the features are linearly independent in $L_2(\Psi)$, so each function
in the span has a unique coefficient vector. The projected Bellman (TD) solution is
defined by
\begin{equation}
v_i^{\mathrm{TD}}=\Pi_\Psi \mathcal{T}_{i,n,\lambda}v_i^{\mathrm{TD}},
\end{equation}
which exists and is unique under Assumption~\ref{ass:well_posed_projected}. The MC
solution is the least-squares return regression
\begin{equation}
w_i^{\mathrm{MC}}=\arg\min_w \E_\Psi[(G_t^{(i)}-v_w(S_t))^2],
\end{equation}
and we write $v_i^{\mathrm{MC}}=v_{w_i^{\mathrm{MC}}}$.

\emph{MC limit.} Decompose $G_t^{(i)}=V_i^\star(S_t)+\varepsilon_t$ with
$\E[\varepsilon_t\mid S_t]=0$, since $V_i^\star(S_t)=\E[G_t^{(i)}\mid S_t]$ by definition of
the GVF. Then, for any $w$,
\[
\E_\Psi[(G_t^{(i)}-v_w(S_t))^2]
=
\E_\Psi[(V_i^\star(S_t)-v_w(S_t))^2]
+
\E_\Psi[\varepsilon_t^2],
\]
because the cross term vanishes by conditioning on $S_t$. Hence
$v_i^{\mathrm{MC}}=\Pi_\Psi V_i^\star$: the MC limit is the $L_2(\Psi)$ projection of
the true GVF onto the feature span.

\emph{Realizable case.} Suppose $V_i^\star=v_{w^\dagger}$ for some $w^\dagger$.
First, $v_i^{\mathrm{MC}}=\Pi_\Psi V_i^\star=V_i^\star$, so $w_i^{\mathrm{MC}}=w^\dagger$
by linear independence. Second, $V_i^\star$ is also the fixed point of
$\mathcal{T}_{i,n,\lambda}$ (claim~(i) of Theorem~\ref{thm:bellman_td_convergence}), so
$\Pi_\Psi\mathcal{T}_{i,n,\lambda}V_i^\star=\Pi_\Psi V_i^\star=V_i^\star$, i.e.,
$V_i^\star$ solves the projected Bellman equation; by the uniqueness in
Assumption~\ref{ass:well_posed_projected}, $v_i^{\mathrm{TD}}=V_i^\star$ as well.
Hence $v_i^{\mathrm{TD}}=v_i^{\mathrm{MC}}=V_i^\star$.

\emph{Non-realizable case.} In general the two limits solve different equations:
$v_i^{\mathrm{MC}}=\Pi_\Psi V_i^\star$ is the direct projection of $V_i^\star$, whereas
$v_i^{\mathrm{TD}}$ solves the composed fixed-point equation
$v=\Pi_\Psi\mathcal{T}_{i,n,\lambda}v$, whose solution is the projection of
$\mathcal{T}_{i,n,\lambda}v_i^{\mathrm{TD}}$ rather than of $V_i^\star$. These agree
only when $\Pi_\Psi V_i^\star$ happens to be a fixed point of
$\Pi_\Psi\mathcal{T}_{i,n,\lambda}$, which fails in general; the projected-Bellman error
bound $\|v_i^{\mathrm{TD}}-V_i^\star\|_\Psi
\le (1-\rho_{i,n,\lambda}^2)^{-1/2}\|\Pi_\Psi V_i^\star-V_i^\star\|_\Psi$ --- the
general-contraction-modulus form of the bound of \citet{tsitsiklis1997analysis}, with
$\rho_{i,n,\lambda}$ the $\mathcal{T}_{i,n,\lambda}$-contraction modulus from
claim~(i) of Theorem~\ref{thm:bellman_td_convergence} in place of the plain-TD(0) discount
$\bar\gamma_i$ --- quantifies the possible gap. Thus without realizability the TD and
MC limits may differ, as claimed.
\qed

\end{document}